\documentclass{article}

\usepackage[utf8]{inputenc} %
\usepackage[T1]{fontenc}    %
\usepackage{hyperref}       %
\usepackage{url}            %
\usepackage{booktabs}       %
\usepackage{graphicx}
\usepackage{amsfonts}
\usepackage{amsmath} %
\usepackage{amssymb} %
\usepackage{nicefrac}       %
\usepackage{microtype}      %
\usepackage{natbib}
\usepackage{bm}
\usepackage{bbold}
\usepackage{algorithmic}

\graphicspath{{./figures/}}

\newcommand{\hone}{\textbf{H1}}

\usepackage{amsthm}
\usepackage{aliascnt}

\newtheorem{proposition}{Proposition}

\usepackage{thm-restate}

\def\bbR{\mathbb{R}}
\def\bbE{\mathbb{E}}

\def\loss{\mathcal{L}}

\DeclareMathOperator*{\argmin}{arg\,min}

\DeclareMathOperator{\Int}{int}
\DeclareMathOperator*{\diag}{Diag}

\title{Super-efficiency of automatic differentiation for functions defined as a minimum}
\author{Pierre Ablin \\ CNRS and DMA \\ Ecole Normale Supérieure - PSL University \\Paris, 75005, France \\ \\
 Gabriel Peyré \\ CNRS and DMA \\ Ecole Normale Supérieure - PSL University\\Paris, 75005, France \\ \\
 Thomas Moreau
 \\ INRIA, CEA \\ Université Paris-Saclay \\ Palaiseau, 91200, France}

\begin{document}

\maketitle

\begin{abstract}
    In min-min optimization or max-min optimization, one has to compute the gradient of a function defined as a minimum.
    In most cases, the minimum has no closed-form, and an approximation is obtained via an iterative algorithm.
    There are two usual ways of estimating the gradient of the function: using either an \emph{analytic} formula obtained by assuming exactness of the approximation, or \emph{automatic} differentiation through the algorithm.
    In this paper, we study the asymptotic error made by these estimators as a function of the optimization error.
    We find that the error of the automatic estimator is close to the square of the error of the analytic estimator, reflecting a \emph{super-efficiency} phenomenon.
    The convergence of the automatic estimator greatly depends on the convergence of the Jacobian of the algorithm.
    We analyze it for gradient descent and stochastic gradient descent and derive convergence rates for the estimators in these cases.
    Our analysis is backed by numerical experiments on toy problems and on Wasserstein barycenter computation.
    Finally, we discuss the computational complexity of these estimators and give practical guidelines to chose between them.
\end{abstract}

\section{Introduction}
In machine learning, many objective functions are expressed as the minimum of another function: functions $\ell$ defined as
\begin{equation}
    \ell(x) = \min_{z \in \bbR^m}\loss\left(z, x\right),
    \label{eq:def_ell}
\end{equation}
where $\mathcal{L}:\bbR^m \times \bbR^n\to\bbR$.
Such formulation arises for instance in dictionary learning, where $x$ is a dictionary, $z$ a sparse code, and $\loss$ is the Lasso cost \citep{Mairal2010}.
In this case, $\ell$ measures the ability of the dictionary $x$ to encode the input data.
Another example is the computation of the Wassertein barycenter of distributions in optimal transport \citep{agueh2011barycenters}: $x$ represents the barycenter, $\ell$ is the sum of distances to $x$, and the distances themselves are defined by minimizing the transport cost.
In the field of optimization, formulation~\eqref{eq:def_ell} is also encountered as a smoothing technique, for instance in reweighted least-squares \citep{daubechies2010iteratively} where $\loss$ is smooth but not $\ell$.
In game theory, such problems naturally appear in two-players maximin games\citep{neumann1928theorie}, with applications for instance to generative adversarial nets \citep{goodfellow2014generative}.
In this setting, $\ell$ should be maximized.

A key point to optimize $\ell$ -- either maximize or minimize -- is usually to compute the gradient of $\ell$, $g^*(x) \triangleq \nabla_x\ell(x)$.
If the minimizer $z^*(x) = \argmin_z\loss(z, x)$ is available, the first order optimality conditions impose that $\nabla_z\loss\left(z^*(x), x\right) = 0$ and the gradient is given by
\begin{equation}
g^*(x) = \nabla_x \loss\left(z^*(x), x \right)\enspace.
    \label{eq:danskin}
\end{equation}
However, in most cases the minimizer $z^*(x)$ of the function is not available in closed-form.
It is approximated via an iterative algorithm, which produces a sequence of iterates $z_t(x)$.
There are then three ways to estimate $g^*(x)$:\\

The \textbf{analytic} estimator corresponds to plugging the approximation $z_t(x)$ in~\eqref{eq:danskin}
    \begin{equation}
    \label{eq:g1}
        g^1_t(x) \triangleq \nabla_x \loss\left(z_t(x), x \right)\enspace.
    \end{equation}

The \textbf{automatic} estimator is $g^2_t(x) \triangleq \partial_x \left[\loss\left(z_t(x), x \right)\right]$, where the derivative is computed with respect to $z_t(x)$ as well. The chain rule gives
\begin{equation}
\label{eq:g2}
    g^2_t(x) = \nabla_x \loss\left(z_t(x), x \right) + \frac{\partial z_t}{\partial x} \nabla_z \loss\left(z_t(x), x \right)
    \enspace .
\end{equation} This expression can be computed efficiently using automatic differentiation \citep{Baydin2018}, in most cases at a cost similar to that of computing $z_t(x)$.\\

If $\nabla_{zz}\loss(z^*(x), x)$ is invertible, the implicit function theorem gives $\frac{\partial z^*(x)}{\partial x} = \mathcal{J}(z^*(x), x)$ where $\mathcal{J}(z, x) \triangleq -\nabla_{xz}\loss\left(z, x\right) \left[\nabla_{zz}\loss\left(z, x\right)\right]^{-1}$. The \textbf{implicit} estimator is
\begin{equation}
\label{eq:g3}
    \medmuskip=1mu \thickmuskip=2mu \thinmuskip=1mu
    g^3_t(x) \triangleq \nabla_x \loss\left(z_t(x), x \right)
        + \mathcal{J}(z_t(x), x)\nabla_z \loss\left(z_t(x), x \right).
\end{equation}
This estimator can be more costly to compute than the previous ones, as a $m \times m$ linear system has to be solved.\\

These estimates have been proposed and used by different communities.
The analytic one corresponds to alternate optimization of $\loss$, where one updates $x$ while considering that $z$ is fixed.
It is  used for instance in dictionary learning \citep{olshausen1997sparse, Mairal2010} or in optimal transport \citep{Feydy2019}.
The second is common in the deep learning community as a way to differentiate through optimization problems \citep{Gregor10}.
Recently, it has been used as a way to accelerate convolutional dictionary learning \citep{Tolooshams2018}.
It has also been used to differentiate through the Sinkhorn algorithm in optimal transport applications \citep{Boursier2019, genevay2018learning}.
It integrates smoothly in a machine learning framework, with dedicated libraries \citep{Abadi2016, Paszke2019}.
The third one is found in bi-level optimization, for instance for hyperparameter optimization \citep{bengio2000gradient}.
It is also the cornerstone of the use of convex optimization as layers in neural networks  \citep{Agrawal2019}. \\

\textbf{Contribution}\quad
In this article, we want to answer the following question:
\emph{which one of these estimators is the best?}
The central result, presented in \autoref{sec:thm},  is the following convergence speed, when $\loss$ is differentiable and under mild regularity hypothesis (\autoref{prop:bound_g1}, \ref{prop:bound_g2} and \ref{prop:bound_g3})
\begin{align*}
    |g^1_t(x) - g^*(x)| &= O\left(|z_t(x)  - z^*(x) |\right)\enspace, \\
    |g^2_t(x) - g^*(x) | &= o\:\left(|z_t(x)  - z^*(x) |\right)\enspace, \\
    |g^3_t(x) - g^*(x) | &= O\left(|z_t(x)  - z^*(x) |^2\right) \enspace .
\end{align*}
This is a super-efficiency phenomenon for the automatic estimator, illustrated in \autoref{fig:illust_intro} on a toy example.
As our analysis reveals, the bound on $g^2$ depends on the convergence speed of the Jacobian of $z_t$, which itself depends on the optimization algorithm used to produce $z_t$.
In \autoref{sec:jaco}, we build on the work of \citet{Gilbert1992} and give accurate bounds on the convergence of the Jacobian for gradient descent (\autoref{prop:jac_gd}) and stochastic gradient descent (\autoref{prop:sgd_cst_step} and \ref{prop:sgd_decr_step}) in the strongly convex case.
We then study a simple case of non-strongly convex problem (\autoref{prop:conv_pth}).
To the best of our knowledge, these bounds are novel.
This analysis allows us to refine the convergence rates of the gradient estimators.
\begin{figure}
    \centering
    \includegraphics[width=.7\linewidth]{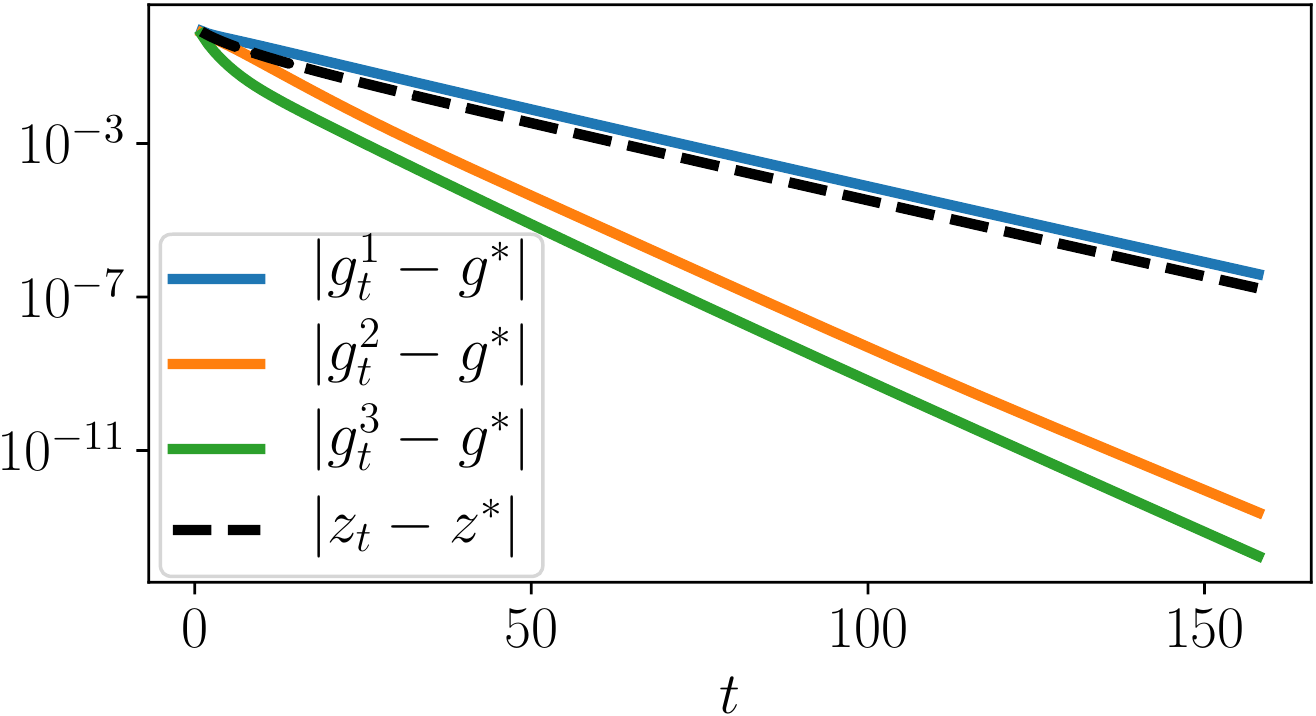}
    \vskip -1em
    \caption{Convergence of the gradient estimators. $\loss$ is strongly convex, $x$ is a random point and $z_t(x)$ corresponds to $t$ iterations of gradient descent. As $t$ increases, $z_t$ goes to $z^*$ at a linear rate. $g_t^1$ converges at the same rate while $g_t^2$ and $g_t^3$ are twice as fast.}
    \label{fig:illust_intro}
\end{figure}
In \autoref{sec:optim}, we start by recalling and extending the consequence of using wrong gradients in an optimization algorithm (\autoref{prop:inexact_gd} and~\ref{prop:inexact_sgd}). Then, since each gradient estimator comes at a different cost, we put the convergence bounds developed in the paper in perspective with a complexity analysis. This leads to practical and principled guidelines about which estimator should be used in which case.
Finally, we provide numerical illustrations of the aforementioned results in \autoref{sec:expe}.\\[.5em]
\textbf{Notation}
\quad
The $\ell^2$ norm of $z \in \bbR^m$ is $|z| = \sqrt{\sum_{i=1}^mz_i^2}$.
The operator norm of $M \in \bbR^{m \times n}$ is $\|M\| = \sup_{|z|=1}|Mz|$ and the Frobenius norm is $\|M\|_F = \sqrt{\sum_{i,j}M_{ij}^2}$.
The vector of size $n$ full of $1'$s is $\mathbb{1}_n$.
The Euclidean scalar product is $\langle \cdot, \cdot\rangle$.

The proofs are only sketched in the article, full proofs are deferred to appendix.
\section{Convergence speed of gradient estimates}
\label{sec:thm}
We consider a compact set $K = K_z \times K_x \subset \bbR^m \times \bbR^n$. We make the following assumptions on $\loss$.\\[.5em]
    \hone:\:$\loss$ is twice differentiable over $K$ with second derivatives $\nabla_{xz}\loss$ and $\nabla_{zz}\loss$ respectively $L_{xz}$ and $L_{zz}$-Lipschitz.\\[.3em]
    \textbf{H2}:\:For all $x \in K_x$, $z \to \loss(z, x)$ has a unique minimizer $z^*(x)\in \Int(K_z)$. The mapping $z^*(x)$ is differentiable, with Jacobian $J^*(x) \in \bbR^{n \times m}$.\\[.5em]
\hone{} implies that $\nabla_z\loss$ and $\nabla_x\loss$ are Lipschitz, with constants $L_z$ and $L_x$.
The Jacobian of $z_t$ at $x\in K_x$ is $J_t \triangleq \frac{\partial z_t(x)}{\partial x} \in \bbR^{n \times m}$.
For the rest of the section, we consider a point $x\in K_x$, and we denote $g^*=g^*(x)$, $z^*=z(x)$ and $z_t=z_t(x)$.
\subsection{Analytic estimator $g^1$}
The analytic estimator~\eqref{eq:g1} approximates $g^*$ \emph{as well as} $z_t$ approximates $z^*$ by definition of the $L_x$-smoothness.
\begin{proposition}[Convergence of the analytic estimator]
\label{prop:bound_g1}
The analytic estimator verifies $
    |g^1_t - g^*| \leq L_x |z_t - z^*| $.
\end{proposition}
\subsection{Automatic estimator $g^2$}
The automatic estimator~\eqref{eq:g2} can be written as
\begin{equation}
\label{eq:g2_dev}
    g^2 = g^* + R(J_t)(z_t - z^*) + R_{xz} + J_t R_{zz}
\end{equation}
where
\begin{align*}
    R_{xz} &\triangleq \nabla_x\loss(z_t, x) - \nabla_x\loss(z^*, x) - \nabla_{xz}\loss(z^*, x)(z_t-z^*) \\
    R_{zz} &\triangleq \nabla_z\loss(z_t, x) - \nabla_{zz}\loss(z^*, x) (z_t - z^*)\enspace.
\end{align*}
are Taylor's rests and
\begin{equation}
    R(J) \triangleq J \nabla_{zz}\loss(z^*, x) + \nabla_{xz}\loss(z^*, x) \enspace.
\end{equation}
The implicit function theorem states that $R(J^*)=0$.
Importantly, in a non strongly-convex setting where $\nabla_{zz}\loss(z^*, x)$ is not invertible, it might happen that $R(J_t)$ goes to $0$ even though $J_t$ does not converge to $J^*$.
\hone{} implies a quadratic bound on the rests
\begin{align}
    \label{eq:bound1}
    |R_{xz}|  \leq \frac{L_{xz}}{2}|z_t - z^*|^2
    \:\text{and }\:
    |R_{zz}|  \leq \frac{L_{zz}}{2}|z_t - z^*|^2.
\end{align}
We assume that $J_t$ is bounded $\|J_t\| \leq L_J$.
This holds when $J_t$ converges, which is the subject of~\autoref{sec:jaco}.
The triangle inequality in \autoref{eq:g2_dev} gives:
\begin{proposition}[Convergence of the automatic estimator]
\label{prop:bound_g2}
We define
\begin{equation}
\label{eq:def_l}
    L \triangleq L_{xz} + L_JL_{zz}
\end{equation}
Then $|g^2_t - g^*| \leq \|R(J_t)\||z_t-z^*| + \frac{L}{2}|z_t-z^*|^2$.

\end{proposition}
This proposition shows that the rate of convergence of $g^2$ depends on the speed of convergence of $R(J_t)$.
For instance, if $R(J_t)$ goes to $0$, we have

$$
g^2_t - g^* = o(|z_t - z^*|)\enspace.
$$
Unfortunately, it might happen that, even though $z_t$ goes to $z^*$, $R(J_t)$ does not go to $0$ since differentiation is not a continuous operation.
In~\autoref{sec:jaco}, we refine this convergence rate by analyzing the convergence speed of the Jacobian in different settings.

\subsection{Implicit estimator $g^3$}
The implicit estimator~\eqref{eq:g3} is well defined provided that $\nabla_{zz}\loss$ is invertible.
We obtain convergence bounds by making a Lipschitz assumption on $\mathcal{J}(z, x) = -\nabla_{xz}\loss(z,x)\left[\nabla_{zz}\loss(z, x)\right]^{-1}$.
\begin{restatable}{proposition}{convergenceImplicit}[Convergence of the implicit estimator]
    \label{prop:bound_g3}
    Assume that $\mathcal{J}$ is $ L_{\mathcal{J}}$-Lipschitz with respect to its first argument, and that $\|\mathcal{J}_t\|\leq L_J$.
    Then, for $L$ as defined in \eqref{eq:def_l},
    \begin{equation}
        |g^3_t - g^*| \leq (\frac{L}{2} + L_{\mathcal{J}}L_{z})|z_t - z^*|^2
        \enspace.
    \end{equation}
\end{restatable}
\begin{proof}[Sketch of proof]
 The proof is similar to that of~\autoref{prop:bound_g2}, using $\|R(\mathcal{J}(z_t, x))\| \leq L_zL_{\mathcal{J}}|z_t -z^*|$.
\end{proof}
Therefore this estimator converges twice as fast as $g^1$, and at least as fast as $g^2$. Just like $g^1$ this estimator does not need to store the past iterates in memory, since it is a function of $z_t$ and $x$.
However, it is usually much more costly to compute.

\subsection{Link with bi-level optimization}

Bi-level optimization appears in a variety of machine-learning problems, such as hyperparameter optimization \citep{pedregosa2016hyperparameter} or supervised dictionary learning \citep{Mairal2012}.
It considers problems of the form
\begin{equation}
    \label{eq:bi_level}
    \min_{x \in \bbR^n}\ell'(x) \triangleq \loss'\left(z^*(x), x\right) \text{ s.t. } z^*(x) \in \argmin_{z\in\bbR^m} \loss(z, x),
\end{equation}
where $\loss': \bbR^m \times \bbR^n \to \bbR$ is another objective function.
The setting of our paper is a special instance of bi-level optimization where $\loss' = \loss$.
The gradient of $\ell'$ is
$$
g'^* = \nabla_x \ell'(x) = \nabla_x\loss'(z^*, x) + J^* \nabla_z\loss'(z^*, x) \enspace.
$$
When $z_t(x)$ is a sequence of approximate minimizers of $\loss$, gradient estimates can be defined as
\begin{align*}
    g'^1 &= \nabla_x \loss'(z_t(x), x), \\
    g'^2 &= \nabla_x \loss'(z_t(x), x) + J_t \nabla_z\loss'(z_t(x), x),  \\
    g'^3 &= \nabla_x \loss'(z_t(x), x) + \mathcal{J}(z_t(x), x) \nabla_z\loss'(z_t(x), x).
\end{align*}
Here, $\nabla_z\loss'(z^*, x)\neq0$, since $z^*$ does not minimize $\loss'$. Hence, $g'^1$ does not estimate $g'^*$.
Moreover, in general,
$$
\nabla_{xz}\loss'(z^*, x) + J^* \nabla_{zz}\loss'(z^*, x) \neq 0.
$$
Therefore, there is no cancellation to allow super-efficiency of $g^2$ and $g^3$ and we only obtain linear rates
$$
g'^2 - g^* = O(|z_t -z^*|), \enspace g'^3 - g^* = O(|z_t - z^*|)\enspace .
$$

\section{Convergence speed of the Jacobian}
\label{sec:jaco}

In order to get a better understanding of the convergence properties of the gradient estimators -- in particular $g^2$ -- we analyze it in  different settings.
A large portion of the analysis is devoted to the convergence of $R(J_t)$ to $0$, since it does not directly follow from the convergence of $z_t$.
In most cases, we show convergence of $J_t$ to $J^*$, and use
\begin{equation}
\label{eq:lip_rj}
    \|R(J_t)\|\leq L_z \|J_t - J^*\|
\end{equation}
in the bound of \autoref{prop:bound_g2}.
\subsection{Contractive setting}
When $z_t$ are the iterates of a fixed point iteration with a contractive mapping, we recall the following result due to \citet{Gilbert1992}.
\begin{proposition}[Convergence speed of the Jacobian]
\label{prop:gilbert}
Assume that $z_t$ is produced by a fixed point iteration
$$z_{t+1} = \Phi\left(z_t, x\right),$$
where $\Phi:K_z\times K_x\to K_z$ is differentiable. We suppose that $\Phi$ is contractive: there exists $\kappa< 1$ such that for all $(z, z', x) \in K_z \times K_z \times K_x$, $
|\Phi(z, x) - \Phi(z', x)| \leq \kappa |z - z'|.
$
Under mild regularity conditions on $\Phi$:
\begin{itemize}\itemsep0em
    \item $z_t$ converges to a differentiable function $z^*$ such that $z^* = \Phi(z^*, x)$, with Jacobian $J^*$.
    \item $|z_t - z^*| = O(\kappa^t)$ and $\|J_t - J^*\| =O(t\kappa^t)$
\end{itemize}
\end{proposition}
\subsection{Gradient descent in the strongly convex case}
We consider the gradient descent iterations produced by the mapping $\Phi(z, x) = z - \rho \nabla_z\loss(z, x)$, with a step-size $\rho \leq 1 / L_z$.
We assume that $\loss$ is $\mu$-strongly convex with respect to $z$, i.e. $\nabla_{zz}\loss \succeq \mu \text{Id}$ for all $z\in K_z, x\in K_x$.
In this setting, $\Phi$ satisfies the hypothesis of \autoref{prop:gilbert}, and we obtain precise bounds.
\begin{restatable}{proposition}{jacGD}[Convergence speed of the Jacobian of gradient descent in a strongly convex setting]
\label{prop:jac_gd}
Let $z_t$ produced by the recursion $z_{t+1} = z_t - \rho \nabla_z\loss(z_t, x)$ with $\rho \leq 1/L_z$ and $\kappa \triangleq 1 - \rho \mu$.
We have $|z_t - z^*| \leq \kappa^t |z_0 - z^*|$ and $\|J_t - J^*\| \leq t \kappa^{t-1} \rho L |z_0 - z^*|$ where $L$ is defined in~\eqref{eq:def_l}.
\end{restatable}
\begin{proof}[Sketch of proof~(\ref{proof:jac_gd})]
We show that $\delta_t = \|J_t - J^*\|$ satisfies the recursive inequality $\delta_{t+1} \leq \kappa \delta_t + \rho L |z_0 - z^*|\kappa^t$.
\end{proof}
As a consequence, Prop.~\ref{prop:bound_g1}, \ref{prop:bound_g2}, \ref{prop:bound_g3} together with Eq.~\eqref{eq:lip_rj} give in this case
\begin{align}
\nonumber
|g^1 - g^*| &\leq L_x |z_0 - z^*| \kappa^t,\\
   |g^2 - g^*| &\leq (\rho L_z t + \frac{\kappa}{2})L |z_0-z^*|^2\kappa^{2t-1}, \label{eq:cvg_gd}\\\nonumber
   |g^3 - g^*|&\leq(\frac L2 + L_{\mathcal{J}}L_z)|z_0 - z^*|^2 \kappa^{2t}\enspace.
\end{align}
We get the convergence speed
$
g^2 - g^* = O(t\kappa^{2t})
$, which is almost twice better than the rate for $g^1$.
Importantly, the order of magnitude in \autoref{prop:jac_gd} is tight, as it can be seen in \autoref{prop:conv_grad_autodiff}.
\subsection{Stochastic gradient descent in $z$}
\label{sec:sgd}
We provide an analysis of the convergence of $J_t$ in the stochastic gradient descent setting, assuming once again the $\mu$-strong convexity of $\loss$. We suppose that $\loss$ is an expectation
$$
\loss(z, x) = \bbE_{\xi}[C(z, x, \xi)]\enspace,
$$
where $\xi$ is drawn from a distribution $d$, and $C$ is twice differentiable. Stochastic gradient descent (SGD)  with steps $\rho_t$ iterates
$$
z_{t+1}(x) = z_t(x) - \rho_t \nabla_zC\left(z_{t}(x), x, \xi_{t+1}\right)\text{ where }\xi_{t+1}\sim d \enspace .
$$
In the stochastic setting, \autoref{prop:bound_g2} becomes
\begin{restatable}{proposition}{sgdBound}
\label{prop:sgd_bound}
Define
\begin{equation}
    \delta_t = \bbE\left[\|J_t - J^*\|_F^2\right] \text{ and } d_t =\bbE\left[|z_t - z^*|^2\right] \enspace.
\end{equation}
We have $ \bbE[|g^2 - g^*|] \leq L_z \sqrt{\delta_t}\sqrt{d_t} + \frac L2 d_t $.
\end{restatable}
\begin{proof}[Sketch of proof~(\ref{proof:sgd_bound})]
We use Cauchy-Schwarz and the norm inequality $\|\cdot\|\leq \|\cdot\|_F$ to bound $\bbE\left[\|R(J_t)\||z_t - z^*|\right]$.
\end{proof}
We begin by deriving a recursive inequality on $\delta_t$, inspired by the analysis techniques of $d_t$.
\begin{restatable}{proposition}{boundingIneq}[Bounding inequality for the Jacobian]
\label{prop:bounding_ineq}
We assume bounded Hessian noise, in the sense that
$\bbE\left[\|\nabla_{zz}C(z, x, \xi)\|_F^2\right] \leq \sigma_{zz}^2$ and $\bbE\left[\|\nabla_{xz}C(z, x, \xi)\|_F^2\right] \leq \sigma_{xz}^2$.
Let $r = \min(n, m)$, and $B^2 = \sigma_{xz}^2 + L_J^2 \sigma_{zz}^2$.
We have
\begin{equation}
    \label{eq:bound_fro}
    \delta_{t+1}\leq (1 - 2 \rho_t \mu) \delta_t + 2 \rho_t \sqrt{r}L \sqrt{d_t}\sqrt{\delta_t} + \rho_t^2 B^2.
\end{equation}
\end{restatable}
\begin{proof}[Sketch of proof~(\ref{proof:bounding_ineq})]
A standard strong convexity argument gives the bound
\begin{equation*}
    \medmuskip0mu\thinmuskip0mu\thickmuskip0mu
\delta_{t+1}\leq
    (1-2\rho_t\mu)\delta_t + 2\rho_t \sqrt{r}L\bbE\left[\|J_t-J^*\|_F|z_t - z_0|\right]
     + \rho_t^2B^2.
\end{equation*}
The middle term is then bounded using Cauchy-Schwarz inequality.
\end{proof}

Therefore, any convergence bound on $d_t$ provides another convergence bound on $\delta_t$ by unrolling Eq.~\eqref{eq:bound_fro}.
We first analyze the fixed step-size case by using the simple ``bounded gradients'' hypothesis and bounds on $d_t$ from \citet{Moulines2011}. In this setting, the iterates converge linearly until they reach a threshold caused by gradient variance.
\begin{restatable}{proposition}{sgdCstStep}[SGD with constant step-size]
\label{prop:sgd_cst_step}
Assume that the gradients have bounded variance $\bbE_{\xi}[|\nabla_z C(z, x, \xi)|^2]\leq \sigma^2$. Assume $\rho_t = \rho<1/L_z$, and let $\kappa_2 = \sqrt{1 - 2\rho\mu}$ and $\beta =\sqrt{\frac{\sigma^2 \rho}{2\mu}}$. In this setting
\begin{equation*}
\delta_t \leq \left(\kappa_2^t\left(\|J^*\|_F + t\alpha\right) + B_2\right)^2\enspace,
\end{equation*}
where $\alpha = \frac{\rho \sqrt{r}L}{\kappa_2}|z^* - z_0|$ and $B_2=\frac{ \rho\sqrt{r}L\beta  }{\kappa_2(1-\kappa_2)} +\frac{\rho B}{(1-\kappa_2)} $.
\end{restatable}
\begin{proof}[Sketch of proof~(\ref{proof:sgd_cst_step})]
\citet{Moulines2011} give $d_t \leq \kappa_2^{2t} |z_0 - z^*|^2 + \beta^2$, which implies
$
    \sqrt{d_t} \leq \kappa_2^t|z_0 - z^*| + \beta\enspace.
$
A bit of work on Eq.~\eqref{eq:bound_fro} then gives
$$
    \sqrt{\delta_{t+1}}\leq\kappa_2\sqrt{\delta_t} + \rho \kappa_2^t + (1 - \kappa_2)B_2
$$
Unrolling the recursion gives the proposed bound.
\end{proof}
This bound showcases that the familiar ``two-stages'' behavior also stands for $\delta_t$: a transient linear decay at rate $(1 - 2\rho\mu)^t$ in the first iterations, and then convergence to a stationary noise level $B_2^2$.
This bound is not tight enough to provide a good estimate of the noise level in $\delta_t$. %
Let $B_3^2$ the limit of the sequence defined by the recursion~\eqref{eq:bound_fro}. We find
$$
B_3^2 = (1 - 2 \rho\mu)B_3^2 + 2 \rho \sqrt{r}L\beta B_3 + \rho^2B^2,
$$
which gives by expanding $\beta$
$$
B_3 = \sqrt{\rho}\frac{\sqrt{r}L\sigma}{(2\mu)^\frac32}\left(1+ \sqrt{1 + \frac{4 \mu^2B^2}{rL^2\sigma^2}}\right)\enspace.
$$
This noise level scales as $\sqrt{\rho}$, which is observed in practice.
In this scenario, Prop.~\ref{prop:bound_g1}, \ref{prop:bound_g2} and \ref{prop:bound_g3} show that $g^1 - g^*$ reaches a noise level proportional to $\sqrt{\rho}$, just like $d_t$, while both $g^2 - g^*$ and $g^3 -g^*$ reach a noise level proportional to $\rho$: $g^2$ and $g^3$ can estimate $g^*$ to an accuracy that can never be reached by $g^1$.
We now turn to the decreasing step-size case.
\begin{restatable}{proposition}{sgdDecrStep}[SGD with decreasing step-size]
\label{prop:sgd_decr_step}
Assume that $\rho_t = \rho_0t^{-\alpha}$ with $\alpha \in (0, 1)$. Assume a bound on $d_t$ of the form $d_t \leq d^2t^{-\alpha}$.
Then
\begin{equation*}
    \delta_t \leq 4 \frac{\rho_0B^2\mu + rL^2d^2}{\mu^2} t^{-\alpha} + o(t^{-\alpha})\enspace.
\end{equation*}
\end{restatable}
\begin{proof}[Sketch of proof~(\ref{proof:sgd_decr_step})]
We use~\eqref{eq:bound_fro} to obtain a recursion
\begin{align*}
\delta_{t+1} \leq \left(1 - \mu \rho_0t^{-\alpha}\right)\delta_t
+ (B^2 \rho_0^2 + \frac{rL^2d^2\rho_0}{\mu})t^{-2\alpha}\enspace,
\end{align*}
which is then unrolled.
\end{proof}
When $\rho_t\propto t^{-\alpha}$, we have $d_t=O(t^{-\alpha})$ so the assumption $d_t \leq d^2 t^{-\alpha}$ is verified for some $d$.
One could use the precise bounds of \citep{Moulines2011} to obtain non-asymptotic bounds on $\delta_t$ as well.

Overall, we recover bounds for $\delta_t$ with the same behavior than the bounds for $d_t$. Pluging them in \autoref{prop:sgd_bound} gives the asymptotic behaviors for the gradient estimators.
\begin{proposition}[Convergence speed of the gradient estimators for SGD with decreasing step]
\label{prop:estimators_sgd}
Assume that $\rho_t = Ct^{-\alpha}$ with $\alpha \in (0, 1)$.
Then
\begin{align*}
    \bbE_{\xi}[|g^1 - g^*|] &=O(\sqrt{t^{-\alpha}}),  \:
    \bbE_{\xi}[|g^2 - g^*|] =O(t^{-\alpha}) \\
   & \bbE_{\xi}[|g^3 - g^*|] =O(t^{-\alpha})
\end{align*}
\end{proposition}
The super-efficiency of $g^2$ and $g^3$ is once again illustrated, as they converge at the same speed as $d_t$.
\subsection{Beyond strong convexity}
\label{sec:least_pth}

All the previous results rely critically on the strong convexity of $\loss$.
A function $f$ with minimizer $z^*$ is $p$-Łojasiewicz \citep{attouch2009convergence}  when $\mu (f(z) - f(z^*))^{p-1} \leq \|\nabla f(z)\|^p $ for some $\mu > 0$.
Any strongly convex function is $2$-Łojasiewicz:
the set of $p$-Łojasiewicz functions for $p\geq 2$ offers a framework beyond strong-convexity that still provides convergence rates on the iterates.
The general study of gradient descent on this class of function is out of scope for this paper. We analyze a simple class of $p$-Łojasiewicz functions, the least mean $p$-th problem, where
\begin{equation}\label{eq:p_loss}
    \loss(z, x) \triangleq \frac{1}{p}\sum_{i=1}^n(x_i -[Dz]_i)^{p}
\end{equation}
for $p$ an even integer and $D$ is overcomplete (rank$(D)=n$). In this simple case, $\loss(\cdot, x)$ is minimized by cancelling $x - Dz$, and $g^* = (x - Dz^*)^{p-1} = 0$.

In the case of least squares ($p=2$) we can perfectly describe the behavior of gradient descent, which converges linearly.
\begin{proposition}
    \label{prop:quadratic}
    Let $z_t$ the iterates of gradient descent with step $\rho \leq \frac{1}{L_z}$ in~\eqref{eq:p_loss} with $p=2$, and $z^* \in \argmin{\loss(z, x)}$.
    It holds
    \begin{align*}
    g^1 = D(z_t - z^*),\:\:
    g^2 = D(z_{2t} - z^*)\:\text{ and }\:
    g^3=0.
    \end{align*}
\end{proposition}
\begin{proof}
The iterates verify $z_t - z^*= (I - D^{\top} D)^t(z_0 - z^*)$, and we find $J_t \nabla_z \loss(z_t, x) = (I_n - (I_n - D^{\top} D)^t)(x - Dz_t)$. The result follows.
\end{proof}
The automatic estimator therefore goes exactly twice as fast as the analytic one to $g^*$, while the implicit estimator is exact.
Then, we analyze the case where $p\geq 4$ in a more restrictive setting.
\begin{restatable}{proposition}{convPth}
\label{prop:conv_pth}
    For $p\geq 4$, we assume $DD^{\top}=I_n$. Let $\alpha \triangleq \frac{p - 1}{p - 2}$. We have
    \begin{align*}
        |g^1_t| = O(t^{-\alpha}), \quad
        |g^2_t| = O(t^{-2\alpha}), \quad
        g^3_t = 0\enspace.
    \end{align*}
\end{restatable}
\begin{proof}[Sketch of proof~(\ref{proof:conv_pth})]
    We first show that the residuals $r_t=x-Dz_t$ verify $r_t = \left(\frac{1}{\rho(p-2)t}\right)^{\frac{1}{p-2}}(1 + O(\frac{\log(t)}{t}))$, which gives the result for $g^1$. We find $g^2_t = M_tr_t^{p-1}$ where $M_t= I_n - J_tD^{\top}$ verifies $M_{t+1} = M_t(I_n - (p-1)\rho \text{diag}(r_t^{p - 2}))$.
    Using the development of $r_t$ and unrolling the recursion concludes the proof.
\end{proof}
For this problem, $g^2$ is of the order of magnitude of $g^1$ squared and as $p\to + \infty$, we see that the rate of convergence of $g^1$ goes to $t^{-1}$, while the one of $g^2$ goes to $t^{-2}$.

\section{Consequence on optimization}
\label{sec:optim}
In this section, we study the impact of using the previous inexact estimators for first order optimization.
These estimators nicely fit in the framework of inexact oracles introduced by \citet{Devolder2014}.
\subsection{Inexact oracle}
\label{sub:inexact_oracle}
We assume that $\ell$ is $\mu_x$-strongly convex and $L_x$-smooth with minimizer $x^*$.
A $(\delta, \mu, L)$-inexact oracle is a couple $(\ell_{\delta}, g_{\delta})$ such that $\ell_{\delta} : \bbR^m \to \bbR$ is the inexact value function, $g_{\delta} : \bbR^m \to \bbR^m$ is the inexact gradient and for all $x, y$
\begin{equation}
    \frac{\mu}{2}|x - y|^2 \le \ell(x) - \ell_{\delta}(y) - \langle g_{\delta}(y) | x - y \rangle \le \frac{L}{2}|x - y|^2 + \delta \enspace .
\end{equation}
\citet{Devolder2013} show that if the gradient approximation $g^i$ verifies $|g^*(x) - g^i(x)| \leq \Delta_i$ for all $x$, then $(\ell, g^i)$ is a $(\delta_i, \frac{\mu_x}{2}, 2L_x)$-inexact oracle, with
\begin{equation}
    \label{eq:inexact_oracle}
    \delta_i = \Delta_i^2(\frac{1}{\mu_x} + \frac{1}{2L_x})~\enspace.
\end{equation}
We consider the optimization of $\ell$ with inexact gradient descent: starting from $x_0\in \bbR^n$, it iterates
\begin{equation}
    \label{eq:gd_x}
    x_{q+1} = x_q - \eta g^i_t(x_q)
    \enspace ,
\end{equation}
with $\eta= \frac{1}{2L_x}$, a fixed $t$ and $i=1, 2$ or $3$.
\begin{restatable}{proposition}{inexactGd}[\citealt[Theorem 4]{Devolder2013}]
    \label{prop:inexact_gd}
    The iterates $x_q$ with estimate $g^i$ verify
    \[
        \ell(x_q) - \ell(x^*) \le 2L_x(1 - \frac{\mu_x}{4 L_x})^q|x_0 - x^*|^2 + \delta_i
    \]
    with $\delta_i$ defined in \eqref{eq:inexact_oracle}.
\end{restatable}

As $q$ goes to infinity, the error made on $\ell(x^*)$ tends towards $\delta_i = \mathcal O (|g^i_t - g^*|^2)$.
Thus, a more precise gradient estimate achieves lower optimization error.
This illustrates the importance of using gradients estimates with an error $\Delta_i$ as small as possible.

We now consider stochastic optimization for our problem, with loss $\ell$ defined as
\[
    \ell(x) = \mathbb E_\upsilon [h(x, \upsilon)]\text{ with } h(x, \upsilon) = \min_z H(z, x, \upsilon)\enspace .
\]
Stochastic gradient descent with constant step-size $\eta\leq \frac{1}{2L_x}$ and inexact gradients iterates
$$
x_{q+1} = x_q - \eta g^i_t(x_q, \upsilon_{q+1})\enspace,
$$
where $g^i_t(x_q, \upsilon_{q+1})$ is computed by an approximate minimization of $z\to H(z, x_q, \upsilon_{q+1})$.
\begin{restatable}{proposition}{inexactSgdAutodiff}
    \label{prop:inexact_sgd}
    We assume that $H$ is  $\mu_x$-strongly convex, $L_x$-smooth and verifies
    \begin{align*}
        \mathbb E [ |\nabla_x h (x, \upsilon) - \nabla_x \ell(x)|^2] \le \sigma^2.
    \end{align*}
    The iterates $x_q$ of SGD with approximate gradient $g^i$ and step-size $\eta$ verify
    \[
        \mathbb E |x_q - x^*|^2 \le (1 - \frac{\eta\mu_x}{2})^q|x_0 - x^*| + \frac{2\eta}{\mu_x} \sigma^2
        + \frac{4}{\mu_x} \delta_i
    \]
    with $\delta_i = \Delta_i^2(\frac{1}{\mu_x} + \frac{1}{2L_x} + 2\eta)$.
\end{restatable}
The proof is deferred to \autoref{proof:inexact_sgd}.
In this case, it is pointless to achieve an estimation error on the gradient $\Delta_i$ smaller than some fraction of the gradient variance $\sigma^2$.

As a final note, these results extend without difficulty to the problem of maximizing $\ell$, by considering gradient ascent or stochastic gradient ascent.
\subsection{Time and memory complexity}
\label{sec:complexity}
In the following, we put our results in perspective with a computational and memory complexity analysis, allowing us to provide practical guidelines for optimization of $\ell$.\\[.5em]
\begin{table}[t]
\centering
\caption{Computational cost for a quadratic loss $\loss$. Here $c\geq 1$ corresponds to the relative added cost of automatic differentiation.}
\label{tab:computational_cost}
\begin{tabular}{|c|c|}
\hline
Gradient estimate & Computational cost   \\ \hline
$g^1_t$           & $\mathcal O(mnt)$            \\ \hline
$g^2_t$           & $\mathcal O(c mn t)$     \\ \hline
$g^3_t$           & $\mathcal O(mnt + m^3+m^2n))$ \\ \hline
\end{tabular}
\end{table}
\textbf{Computational complexity of the estimators}
\quad
The cost of computing the estimators depends on the cost function $\loss$.
We give a complexity analysis in the least squares case~\eqref{eq:p_loss} which is summarized in \autoref{tab:computational_cost}.
In this case, computing the gradient $\nabla_z \loss$ takes $\mathcal{O}(mn)$ operations, therefore the cost of computing $z_t$ with gradient descent is $\mathcal{O}(mnt)$.
Computing $g^1_t$ comes at the same price.
The estimator $g^2$ requires a reverse pass on the computational graph, which costs a factor $c \ge 1$ of the forward computational cost:
the final cost is $\mathcal{O}(c mn)$.
\citet{griewank2008evaluating} showed that typically $c \in [2, 3]$.
Finally, computing $g^3$ requires a costly $\mathcal{O}(m^2n)$ Hessian computation, and a $\mathcal{O}(m^3)$ linear system inversion. The final cost is $\mathcal{O}(mnt + m^3 + m^2n)$.
The linear scaling of $g^1_t$ and $g^2_t$ is highlighted in \autoref{fig:time_complexity}.
In addition, computing $g^2$ usually requires to store in memory all intermediate variables, which might be a burden.
However, some optimization algorithms are \emph{invertible}, such as SGD with momentum \citep{maclaurin2015gradient}.
In this case, no additional memory is needed.\\[.5em]
\textbf{Linear convergence: a case for the analytic estimator}\quad
In the time it takes to compute $g^2_t$, one can at the same cost compute $g^1_{c t}$.
If $z_t$ converges linearly at rate $\kappa^t$, \autoref{prop:bound_g1} shows that $g^1_{c t} - g^* = O(\kappa^{c t})$, while \autoref{prop:bound_g2} gives, at best, $g^2_{t} - g^* = O(\kappa^{2t})$: $g^1_{c t}$ is a better estimator of $g^*$ than $g^2_t$, provided that $c \geq 2$. In the quadratic case, we even have $g^2_t = g^1_{2t}$.
Further, computing $g^2$ might requires additional memory: $g^1_{ct}$ should be preferred over $g^2_t$ in this setting.
However, our analysis is only asymptotic, and other effects might come into play to tip the balance in favor of $g^2$.

As it appears clearly in~\autoref{tab:computational_cost}, choosing $g^1$ over $g^3$ depends on $t$: when $mnt \gg m^3 + m^2n$, the additional cost of computing $g^3$ is negligible, and it should be preferred since it is more accurate.
This is however a rare situation in a large scale setting.\\[.5em]
\textbf{Sublinear convergence}
\quad
We have provided two settings where $z_t$ converges sub-linearly.
In the stochastic gradient descent case with a fixed step-size, one can benefit from using $g^2$ over $g^1$, since it allows to reach an accuracy that can never be reached by $g^1$.
With a decreasing step-size, reaching $|g^1_t - g^*| \leq \varepsilon$ requires $O(\varepsilon^{-2/\alpha})$ iterations, while reaching $|g^2_t - g^*|\leq \varepsilon$ only takes $O(\varepsilon^{-1/\alpha})$ iterations.
For $\varepsilon$ small enough, we have $c\varepsilon^{-1 / \alpha}<\varepsilon^{-2/\alpha}$: it is always beneficial to use $g^2$ if memory capacity allows it.

The story is similar for the simple non-strongly convex problem studied in \autoref{sec:least_pth}: because of the slow convergence of the algorithms, $g^2_t$ is \emph{much} closer to $g^*$ than $g^1_{ct}$.
Although our analysis was carried in the simple least mean $p$-th problem, we conjecture it could be extended to the more general setting of $p$-Łojasiewicz functions \citep{attouch2009convergence}.

\begin{figure*}[pt!]
    \centering
    \makebox[\textwidth][c]{\includegraphics[width=1.6\linewidth]{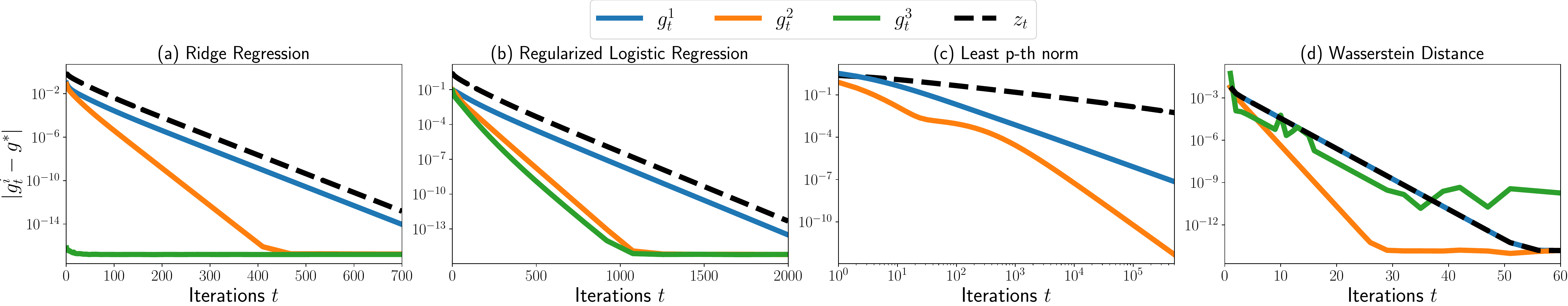}}

    \caption{
        Evolution of $|g^i_t - g^*|$ with the number of iteration $t$ for (\emph{a}) the Ridge Regression $\loss_1$, (\emph{b}) the Regularized Logistic Regression $\loss_2$, (\emph{c}) the Least Mean $p$-th Norm $\loss_3$ in \emph{log-scale} and (\emph{d}) the Wasserstein Distance $\loss_4$. In all cases, we can see the asymptotic super-efficiency of the $g_2$ estimator compared to $g_1$. The $g_3$ estimator is better in most cases but it is unstable in (\emph{d}).
    }
    \label{fig:gd:super_efficiency}
\end{figure*}

\section{Experiments}
\label{sec:expe}

All experiments are performed in Python using \texttt{pytorch} \citep{Paszke2019}. The code to reproduce the figures is available online.\footnote{See Appendix.}

\subsection{Considered losses}

In our experiments, we considered several losses with different properties. For each experiments, the details on the size of the problems are reported in \autoref{app:exp_details}.\\[.5em]
\textbf{Regression}\quad
For a design matrix $D \in \bbR^{n\times m}$ and a regularization parameter $\lambda >0$, we define
\begin{align*}
    \loss_1(z, x) =& \frac12|x - Dz|^2 + \frac{\lambda}{2} |z|^2
    \enspace ,\\
    \loss_2(z, x) =& \sum_{i=1}^n \log\left(1 + e^{-x_i [Dz]_i}\right)
        + \frac{\lambda}{2} |z|^2
    \enspace ,\\
    \loss_3(z, x) =& \frac{1}{p}| x- Dz|^{p}; ~~~ p = 4
    \enspace .
\end{align*}
$\loss_1$ corresponds to Ridge Regression, which is quadratic and strongly convex when $\lambda > 0$.
$\loss_2$ is the Regularized Logistic Regression. It is strongly convex when $\lambda > 0$.
$\loss_3$ is studied in \autoref{sec:least_pth}, and defined with $DD^{\top} = I_n$.\\[.5em]
\textbf{Regularized Wasserstein Distance}
\quad
The Wasserstein distance defines a distance between probability distributions.
In \citet{cuturi2013sinkhorn}, a regularization of the problem is proposed, which allows to compute it efficiently using the Sinkhorn algorithm, enabling many large scale applications.
As we will see, the formulation of the problem fits nicely in our framework.
The set of histograms is $\Delta_+^m = \{a \in \bbR_+^m | \enspace \sum_{i=1}^ma_i=1\}$. Consider two histograms $a \in \Delta^{m_a}_+$ and $b\in \Delta^{m_b}_+$.
The set of couplings is $U(a, b) = \{P \in \bbR_+^{m_a\times m_b} | \enspace P\mathbb{1}_{m_b} = a, \enspace P^{\top}\mathbb{1}_{m_a} = b\}$.
The histogram $a$ (resp. $b$) is associated with set of $m_a$ (resp. $m_b$) points in dimension $k$, $(X_1, \dots, X_{m_a}) \in \bbR^k$ (resp $(Y_1, \dots, Y_{m_b})$).
The cost matrix is $C\in\bbR^{m_a \times m_b}$ such that $C_{ij} = |X_i - Y_j|^2$.
For $\epsilon > 0$, the entropic regularized Wasserstein distance is $W^2_{\epsilon}(a, b) = \min_{P\in U(a, b)} \langle C, P\rangle + \epsilon\langle \log(P), P \rangle$.
The dual formulation of the previous variational formulation is~\citep[Prop. 4.4.]{Peyre2019}:
\begin{equation}
    \label{eq:wasserstein}
    W^2_{\epsilon}(a, b) = \min_{z_a, z_b} \underbrace{\langle a, z_a\rangle + \langle b, z_b\rangle + \epsilon \langle e^{-z_a/\epsilon}, e^{-C/\epsilon} e^{-z_b/ \epsilon}\rangle}_{\loss_4\left((z_a, z_b), a\right)}
\end{equation}
This loss is strongly convex up to a constant shift on $z_a,z_b$.
The Sinkhorn algorithm performs alternate minimization of $\loss_4$ :
\begin{align*}
z_a&\leftarrow \epsilon (\log(e^{-C/ \epsilon}e^{-z_b/\epsilon}) - \log(a)),\\
z_b&\leftarrow \epsilon (\log(e^{-C^{\top}/ \epsilon}e^{-z_a/\epsilon}) - \log(b))\enspace.
\end{align*}
This optimization technique is not covered by the results in \autoref{sec:jaco}, but we will see that the same conclusions hold in practice.
\subsection{Examples of super-efficiency}
To illustrate the tightness of our bounds, we evaluate numerically the convergence of the different estimators $g^1, g^2$ and $g^3$ toward $g^*$ for the losses introduced above. For all problems, $g^*$ is computed by estimating $z^*(x)$ with gradient descent for a very large number of iterations and then using~\eqref{eq:danskin}.\\[.5em]
\textbf{Gradient Descent}\quad
\autoref{fig:gd:super_efficiency} reports the evolution of $|g^i_t - g^*|$ with $t$ for the losses $\{\loss_j\}_{j=1}^4$, where $z_t$ is obtained by gradient descent for $\loss_1, \loss_2$ and $\loss_3$, and by Sinkhorn iterations for $\loss_4$.
For the strongly convex losses (\emph{a}),(\emph{b}), $|g^1_t - g^*|$ converges linearly with the same rate as $|z_t - z^*|$ while $|g^2_t - g^*|$ converges about twice as fast.
This confirms the theoretical findings of \autoref{prop:gilbert} and \eqref{eq:cvg_gd}.
The estimator $g^3$ also converges with the predicted rates in (\emph{a}),(\emph{b}), however, it fails in (\emph{d}) as the Hessian of $\loss_4$ is ill-conditionned, leading to numerical instabilities.
For the non-strongly convex loss $\loss_3$, \autoref{fig:gd:super_efficiency}.(c) shows that  the rates given in \autoref{prop:conv_pth} are correct as $g_1$ converges with a rate $t^{-\frac{3}{2}}$ while $g^2_t$ converges as $t^{-3}$. Here, we did not include $g^3_t$ as it is equal to $0$ due to the particular form of $\loss_3$.\\[.5em]
\begin{figure}
    \centering
    \includegraphics[width=\linewidth]{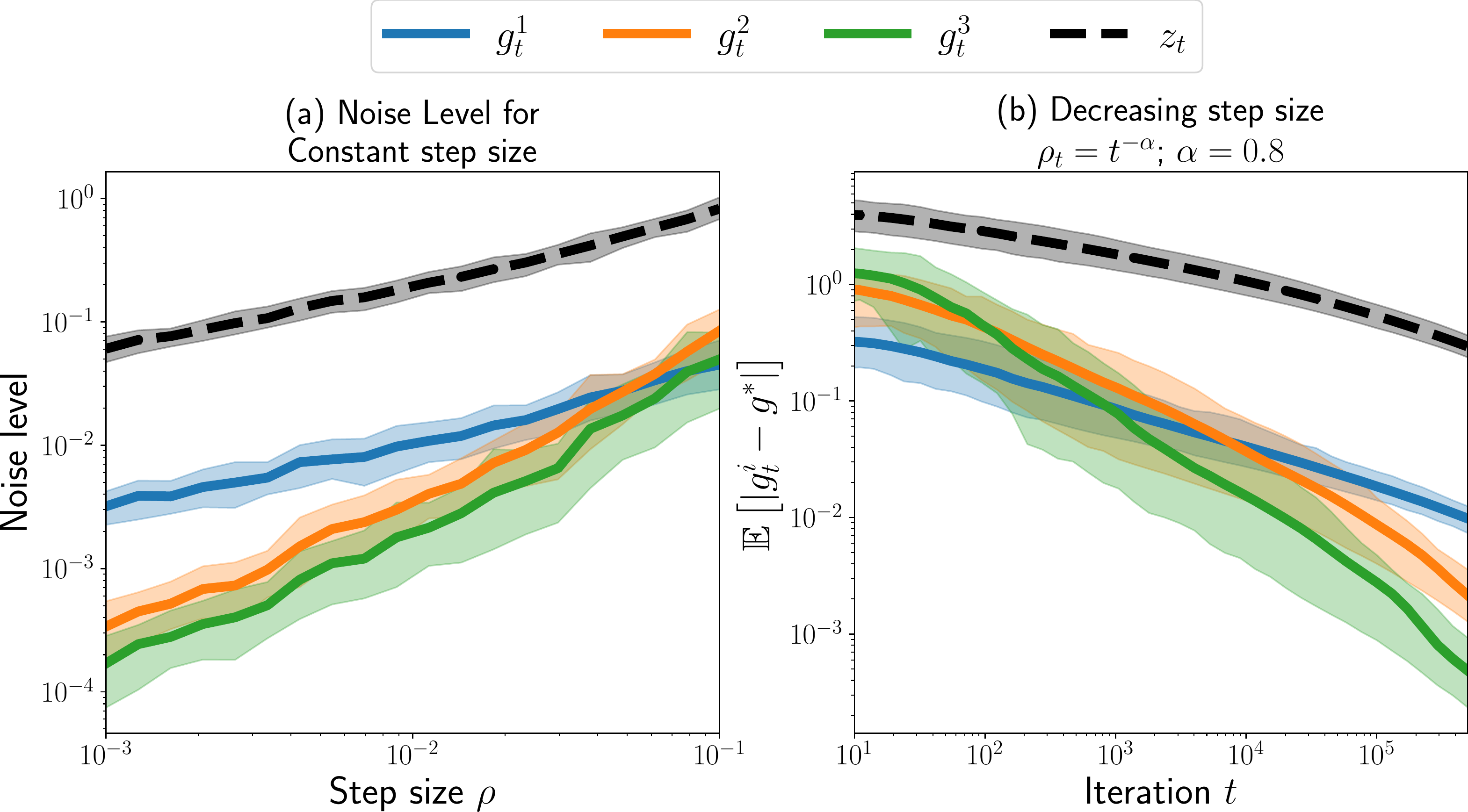}
    \vskip-1em
    \caption{Expected performances of $g^i$ for the SGD; (\emph{a}) noise level as $t\to +\infty$ for a constant step-size $\rho$; (\emph b) Expected error as a function of the number of iteration for decreasing step-size $\rho_t = Ct^{-0.8}$. The solid line displays the mean values and the shaded area the first and last decile.}
    \label{fig:sgd:noise_level}
\end{figure}
\textbf{Stochastic Gradient Descent}\quad
In \autoref{fig:sgd:noise_level}, we investigate the evolution of expected performances of $g^i$ for the SGD, in order to validate the results of \autoref{sec:sgd}. We consider $\loss_2$.
The left part (\emph a) displays the asymptotic expected performance $\mathbb E[|g^i_t - g^*|]$ in the fixed step case, as a function of the step $\rho$, computed by running the SGD with sufficiently many iterations to reach a plateau. As predicted in \autoref{sec:sgd}, the noise level scales as $\sqrt{\rho}$ for $g^1$ while it scales like $\rho$ for $g^2$ and $g^3$.
The right part (\emph b) displays the evolution of $\mathbb E[|g^i_t - g^*|]$ as a function of $t$, where the step-size is decreasing $\rho_t\propto t^{-\alpha}$. Here again, the asymptotic rates predicted by \autoref{prop:estimators_sgd} is showcased: $g^1 - g^*$ is $O(\sqrt{t^{-\alpha}})$ while $g^2 - g^*$ and $g^3 - g^*$ are $O(t^{-\alpha})$.

\subsection{Example on a full training problem}
We are now interested in the minimization of $\ell$ with respect to $x$, possibly under constraints. We consider the problem of computing Wasserstein barycenters using mirror descent, as proposed in \citep{cuturi2014fast}.
For a set of histograms $b_1,\dots, b_N \in \Delta^{m}_+$ and a cost matrix $C\in \bbR^{n\times m}$, the entropic regularized Wasserstein barycenter of the $b_i$'s is
\[
    x\in\argmin_{x\in\Delta^{n}_+} \ell(x) =  \sum_{i=1}^NW^2_{\epsilon}(x,b_i) \enspace,
\]
where $W^2_{\epsilon}$ is defined in \eqref{eq:wasserstein}, and we have:
\begin{equation}
    \label{eq:wasserstein_barycenter}
    \ell(x) = \min_{z_x^1, \dots, z_x^N, z_b^1, \dots, z_b^N}
        \sum_{i=1}^N \loss_4\left((z_x^i, z_b^i), x\right)\enspace.
\end{equation}
The dual variables $z_x^i, z_b^i$ are obtained with $t$ iterations of the Sinkhorn algorithm. In this simple setting, $\nabla_x\loss_4\left((z_x, z_b), x\right) = z_x$.
The cost function is then optimized by mirror descent, with approximate gradient $g^i$: $x_{q+1} = P_{\Delta}(\exp(-\eta g^i)x_q)$, where $P_{\Delta}(x) = x / \sum_{i=1}^n x_i$ is the projection on $\Delta^n_+$.
\begin{figure}[t]
    \centering
    \includegraphics[width=\linewidth]{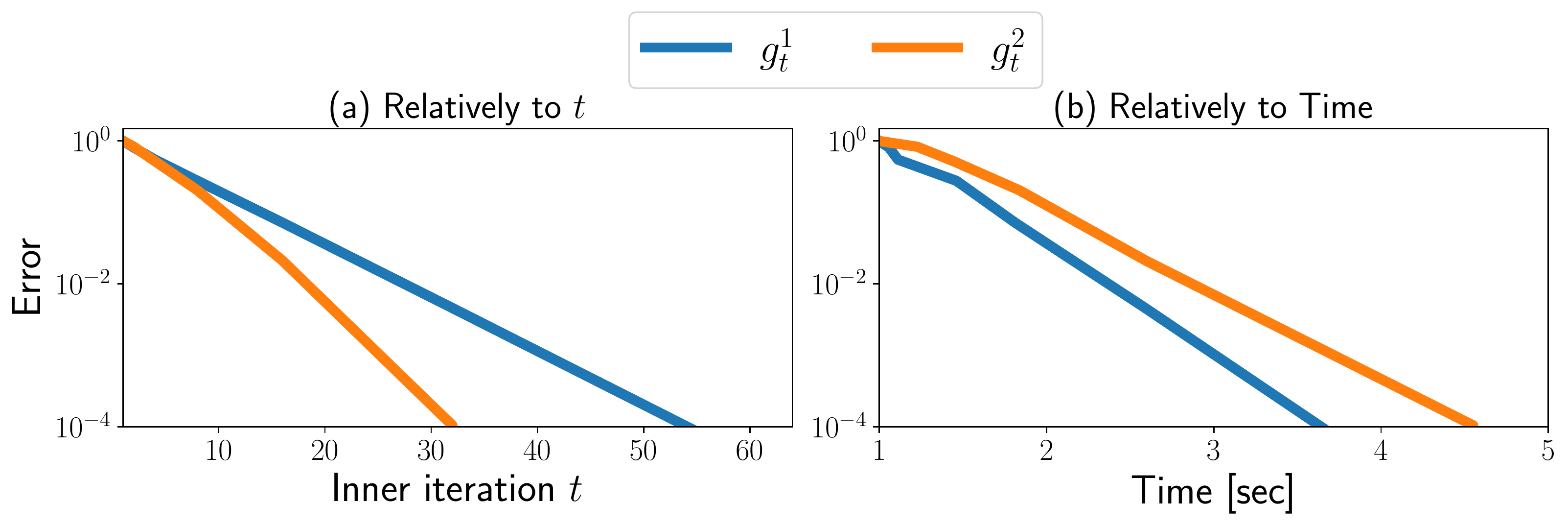}
    \vskip-1em
    \caption{Final optimization error $\delta_i$ relatively to (\emph a) the number of inner iterations used to estimation $g^i$; (\emph b) the time taken to reach this optimization error level.}
    \label{fig:optim:barycenter}
\end{figure}
\autoref{fig:optim:barycenter} displays the scale of the error $\delta_i = \ell(x_q) - \ell(x^*)$. We excluded $g^3$ here as the computation were unstable -- as seen in \autoref{fig:gd:super_efficiency}.(c) -- and too expensive. The error decreases much faster with number of inner iteration $t$ by using $g^2_t$ compared to $g^1_t$. However, when looking at the time taken to reach the asymptotic error, we can see that $g^1$ is a better estimator in this case. This illustrates the fact that while $g^2$ is almost twice as good at approximating $g^*$ as $g^1$, it is at least twice as expensive, as discussed in \autoref{sec:complexity}.

\section*{Conclusion}
In this work, we have described the asymptotic behavior of three classical gradient estimators for a special instance of bi-level estimation.
We have highlighted a super-efficiency phenomenon of automatic differentiation.
However, our complexity analysis shows that it is faster to use the standard analytic estimator when the optimization algorithm converges linearly, and that the super-efficiency can be leveraged for algorithms with sub-linear convergence.
This conclusion should be taken with caution, as our analysis is only asymptotic.
This suggests a new line of research interested in the non-asymptotic behavior of these estimators.
Extending our results to a broader class of non-strongly convex functions would be another interesting direction, as we observe empirically that for logistic-regression, $g_2 - g^* \simeq (g_1 - g^*)^2$.
However, as the convexity alone does not ensure the convergence of the iterates, it raises interesting question for the gradient estimation.
Finally, it would also be interesting to extend our analysis to non-smooth problems, for instance when $z_t$ is obtained with the proximal gradient descent algorithm as in the case of ISTA for dictionary learning.

\section*{Acknowledgement}
P.A. and G.P. acknowledge support from the European Research Council (ERC-NORIA). This work was funded in part by the French government under management of Agence Nationale de la Recherche as part of the "Investissements d'avenir" program, reference ANR-19-P3IA-0001 (PRAIRIE 3IA Institute).

\bibliographystyle{plainnat}
\bibliography{library}

\newpage
~
\newpage
\appendix

\renewcommand\thefigure{\thesection.\arabic{figure}}
\setcounter{figure}{0}

\section{Experiments details and extra experiments}

\subsection{Experiments details}
\label{app:exp_details}

\textbf{Super-efficiency of $g^2$ for gradient descent}
\quad
For \autoref{fig:gd:super_efficiency}, the problem sizes are:

\begin{itemize}
    \item \textbf{Ridge regression $\loss_1$}: we use an overcomplete design matrix $D$ with $n=50$ and $m=100$ with entries $D_{i,j}$ drawn \emph{iid} from a normal distribution $\mathcal N(0, 1)$. The vector $x$ to evaluate the gradient is sampled also with \emph{iid} entries following a normal distribution. We take $\lambda = \frac 1n$. To compute $g^*(x)$, we used the gradient descent with step-size $\frac 1L_z$ for $14,000$ iterations.
    \item \textbf{Regularized Logistic regression $\loss_2$}: we use an overcomplete design matrix $D$ with $n=50$ and $m=100$ with entries $D_{i,j}$ drawn \emph{iid} from a normal distribution $\mathcal N(0, 1)$. The vector $x$ to evaluate the gradient is sampled also with \emph{iid} entries following a normal distribution. We take $\lambda = \frac 1n$. To compute $g^*(x)$, we used the gradient descent with step-size $\frac 1L_z$ for $40,000$ iterations.
    \item \textbf{Least mean $p$-th norm $\loss_3$}: In this setting, the convergence is much slower than in the previous ones. We use an overcomplete design matrix $D$ with $n=5$ and $m=10$. To meet the condition of \autoref{prop:conv_pth}, we sample the entries of a matrix $A \in\bbR^{n\times m}$ \emph{iid} with normal distribution $\mathcal N(0, 1)$, take the SVD of $A = U^\top\Lambda V$ with $U\in\bbR^{n\times n}$ unitary and $V$ and define $D = U^\top V$. This ensures that $DD^\top = I_n$ We choose $p=4$, and use $g^* = 0$.
    \item \textbf{Wasserstein Distance $\loss_4$}: we consider here the problem of computing the Wasserstein distance between two distributions supported on an Euclidean grid in $[0, 1]$ and $C$ is defined as the $\ell_2$ distance between the points of the grid. $q$ is in $\Delta_+^n$ with $n=100$ and $b \in \Delta_+^m$ with $m=30$. We sample $a\in\Delta_+^m$ by first sampling the $\widehat a$ \emph{iid} from a uniform distribution $\mathcal U(0, 1)$ and then take $a = \frac{\widehat a}{\sum_{l=1}^ma_l|}$. We used $\epsilon = 0.1$ and $g^*(x)$ is computed by running $2,000$ iteration of Sinkhorn.
\end{itemize}

\textbf{Super-efficiency of $g^2$ for SGD}
\quad
For \autoref{fig:sgd:noise_level}, we consider for both experiments the penalized logistic loss and an over-complete design matrix $D$ with $n=30$ and $m=50$ with entries $D_{i,j}$ drawn \emph{iid} from a normal distribution $\mathcal N(0, 1)$. The vector $x$ to evaluate the gradient is sampled also with \emph{iid} entries following a normal distribution. We take $\lambda = \frac 1n$. To compute $g^*(x)$, we used the gradient descent with step-size $\frac 1{L_z}$ for $1,000$ iterations.

In \autoref{fig:sgd:noise_level}.(a), we compute the gradient estimates $g^i_t$ using the output $z_t$ of the SGD with constant step-size $\rho$ for 20 values of $\rho \in [0.001, 0.1]$ in log-scale for $50$ realization of the SGD.
We report the mean value of $|g^i_t - g^*|$ computed for $t$ large enough to have reached the regime where only the noise term is significant in \autoref{prop:sgd_cst_step}.
This corresponds to the value of $\mathbb E |g^i_t - g^*|$ estimated by taking the value at the right end of the curve displayed in \autoref{fig:sgd:gradient_estimation}.(a) for different values of $\rho$.

\begin{figure}[tp!]
    \centering
    \includegraphics[width=\linewidth]{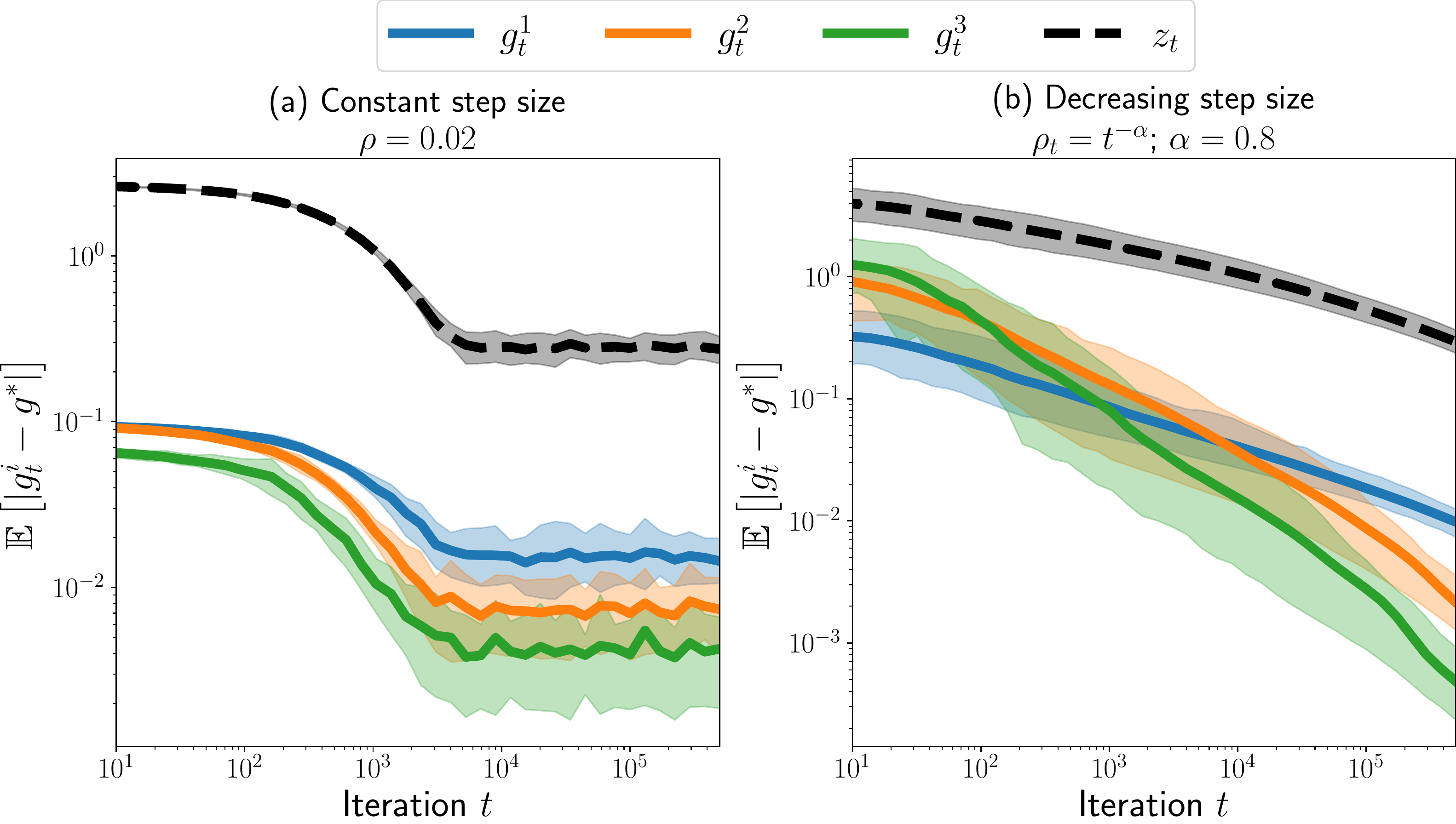}
    \vskip-1em
    \caption{Evolution of $\mathbb E | g^i - g^*|$ with the number of iterations $t$ for (\emph{a}) SGD with a constant step-size $\rho=0.02$; (\emph{b}) SGD with a decreasing step-size $\rho_t = t^{-\alpha}$ for $\alpha = 0.8$.}
    \label{fig:sgd:gradient_estimation}
\end{figure}

In \autoref{fig:sgd:noise_level}.(b), we illustrate the evolution with $t$ of $\mathbb E|g^i_t-g^*|$ for $g^i_t$ computed for SGD with decreasing step-sizes $t^{-\alpha}$ for $\alpha = .8$.
The expectation is estimated by averaging $50$ realizations of $|g^i_t - g^*|$ and the area around the curve correspond to the first and $9$-th deciles.

\subsection{Gradient descent with inexact gradients}
\label{app:inexact_gd}

\begin{figure*}[t]
    \centering
    \includegraphics[width=\textwidth]{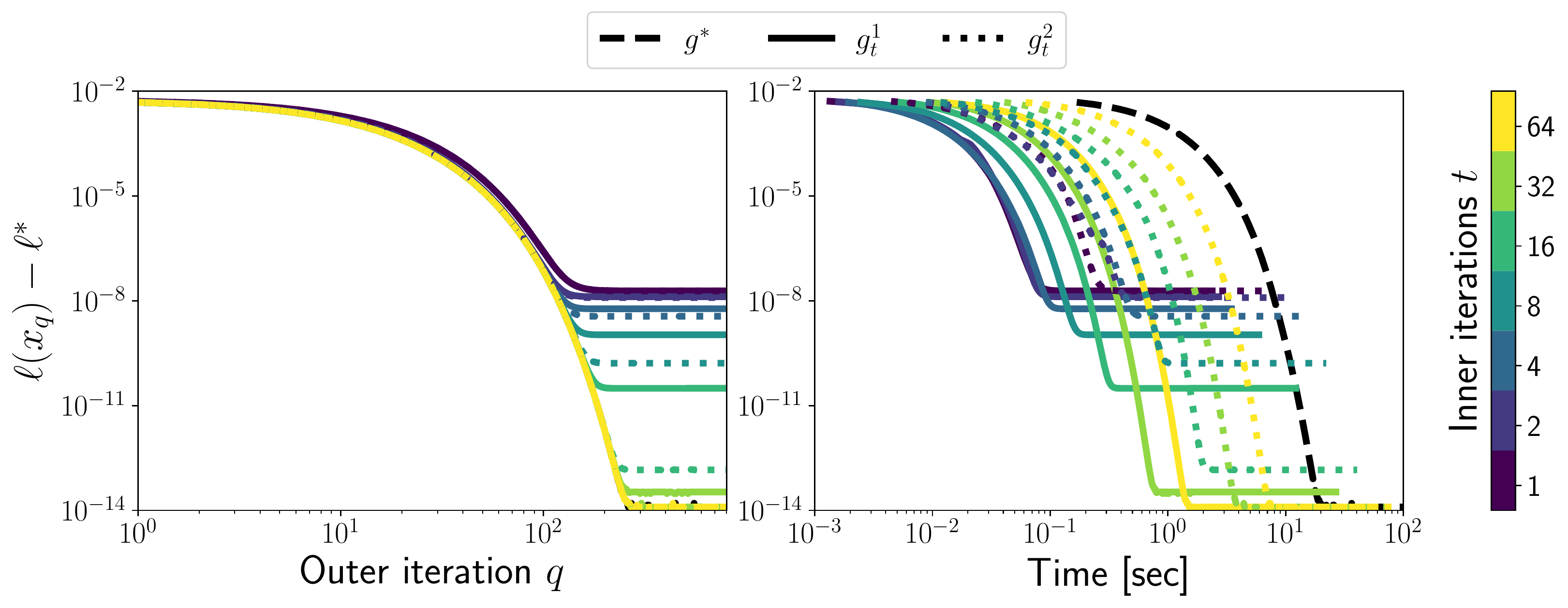}
    \vskip-1em
    \caption{Evolution of $\ell(x_q) - \ell*$ with $q$ for different values of $t$ and for analytic estimator $g^1_t$ and automatic estimator $g^2_t$.}
    \label{fig:optim:barycenter_loss}
\end{figure*}

To evaluate the impact of the different gradient estimators on the optimization of the global function $\ell$, we run mirror descent for the loss $\ell$ defined in \eqref{eq:wasserstein_barycenter}.
We use $n=m=1,000$ and $N=20$ for the dimensions of $x\in\Delta_+^n$ and $b_i \in \Delta_+^m$ and we sample following the same procedure as for the Wasserstein Distance.
We set $\epsilon = 0.05$ and we used a step-size of $\eta = 0.05$.
We compute $x^*$ by running the mirror descent algorithm with analytic gradient estimator $g^1_t$ for $t=1,000$ and $q = 5,000$.
\autoref{fig:time_complexity} reports the residual errors $\ell(x_q) - \ell(x^*)$ for $g^1_t$ and $g^2_t$ relatively to the number of iterations $t$ used to compute them (\emph{a}) as well as to the time taken to compute them (\emph{b}).
We exclude $g^3$ from this analysis as it is much more costly to compute in this case (see \autoref{app:computational_time}) and it can be ill-conditionned -- as it is illustrated in \autoref{fig:gd:super_efficiency}.
\autoref{fig:optim:barycenter_loss} displays for each $t$ used to compute \autoref{fig:optim:barycenter} the evolution of the cost function in iteration $q$ and in time.
We can see here in both figures that as the number of iteration $t$ to compute the gradient increases, the final optimization error decrease, as predicted in \autoref{prop:inexact_gd}. However, the computational cost for $g^2_t$ scales with a factor $c$ compared to computing $g^1_t$ and $c$ is larger than 2 in this case (see \autoref{app:computational_time}).
As the convergence of $z_t$ is linear, using $g^2_t$ is not beneficial for the global optimization as it possible to compute $g^1_{ct}$ instead which reduces the optimization error compared to $g^1_t$, as discussed in \autoref{sec:complexity}.

\begin{figure}[t]
    \centering
    \includegraphics[width=.6\linewidth]{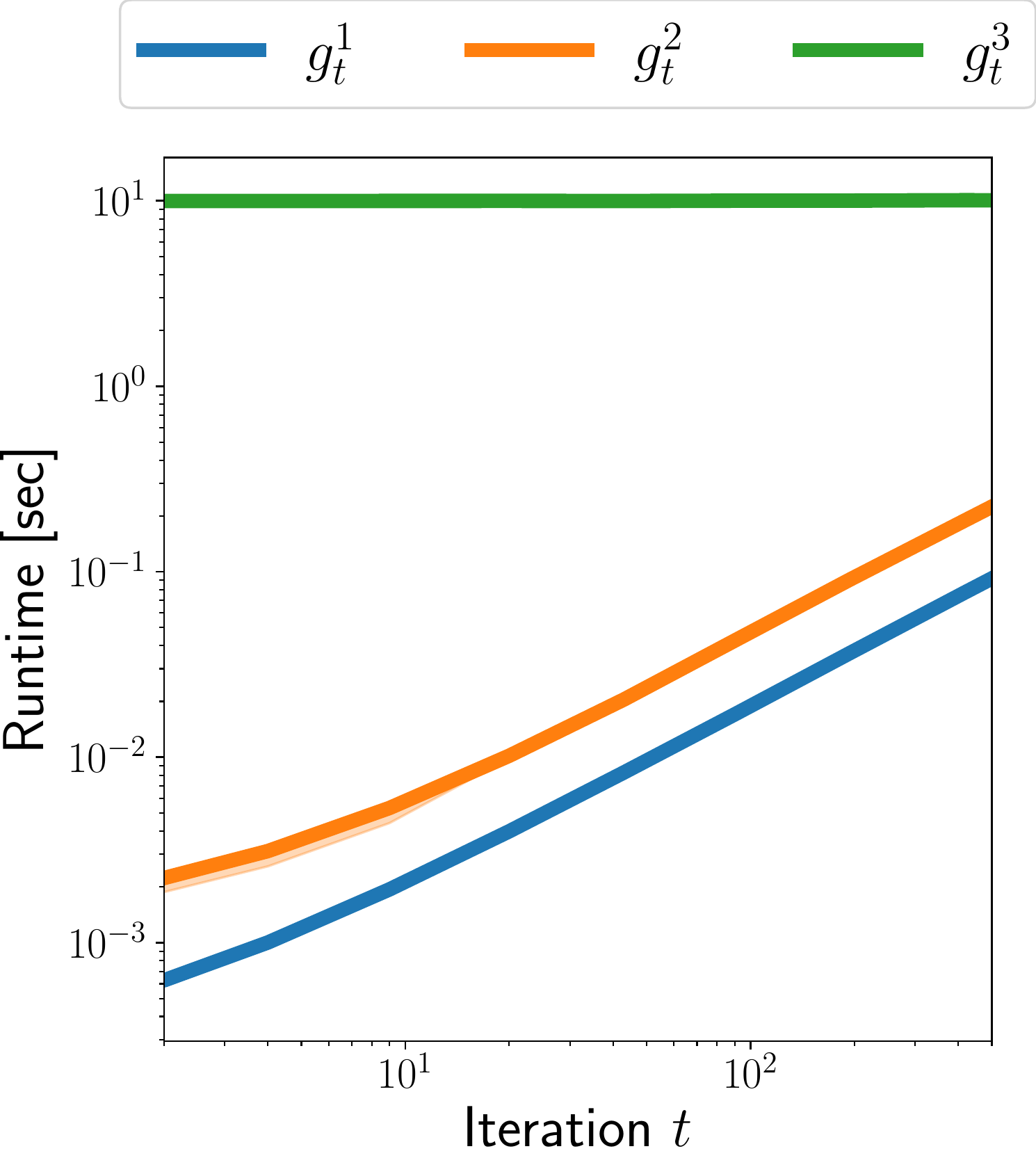}
    \vskip-1em
    \caption{Evolution of the computation time for the gradient estimators $g^i_t$ with the number of iteration $t$.}
    \label{fig:time_complexity}
\end{figure}

\subsection{Computation time of the gradient}
\label{app:computational_time}

To evaluate the relative computational cost of the gradient estimates $g^i_t$, we time the computation of the gradient using the loss $\ell$ defined in \eqref{eq:wasserstein_barycenter}.
We use $n=500$, $m=1,000$ and $N=100$ for the dimensions of $x\in\Delta_+^n$ and $b_i \in \Delta_+^m$ and we sample following the same procedure as for the Wasserstein Distance.
Then, we time the computational time for the gradient estimators $g^i_t$ for different values of $t$, and report in \autoref{fig:time_complexity} the median value of this computation time computed on $50$ realization as well as the first and last decile values to get an idea of the variation of this value.
The results are coherent with the computational complexity analysis in \autoref{tab:computational_cost}, $g^1$ and $g^2$ computation time scales linearly with $t$, with a constant factor between them which capture the value of $c$ which is around 3.5 in our case.
For this scale of problem, $g^3$ requires inverting $N$ matrices $n\times n$. This cost dominates the cost of computing $z_t$, $g^1_t$ and $g^2_t$ for small value of $t$ and it becomes less prohibitive as $t$ grows.

\section{Proof for \autoref{sec:thm}}

\convergenceImplicit*

\begin{proof}
 We define $\mathcal{J}_t=\mathcal{J}(z_t, x)$. The implicit gradient $g^3$ reads
\[
    g^3 = g^* + R(\mathcal{J}_t)(z_t - z^*) + R_ {xz}
            + \mathcal{J}_tR_{zz}
    \enspace.
\]
 Then, we have
 \[
    R(\mathcal{J}_t) = R(\mathcal{J}_t) - R(J^*) = \left(\mathcal{J}_t  - J^*\right)\nabla_{zz}\loss(z^*, x)
\]
We recall that $J^* = \mathcal J(z^*, x)$.
It follows that
\[
    \|\mathcal{J}_t  - J^*\|\leq L_{\mathcal{J}}|z_t - z^*|
     \enspace,
\]
and
\[
    \|R(\mathcal{J}_t)\| \leq L_{\mathcal{J}}L_{z}|z_t-z^*|\enspace.
\]
The result follows using \autoref{eq:bound1}.
\end{proof}

\section{Proof \autoref{sec:jaco}}
\label{proof:sec:jaco}

\subsection{Proof of~\autoref{prop:jac_gd}}
\label{proof:jac_gd}

\jacGD*{}

\begin{proof}
Differentiating the gradient descent recursion, we find that $J_t$ follows the recursion
\begin{equation}
    \label{eq:jac_rec}
    J_{t+1} = J_t - \rho G_t\enspace,
\end{equation}
where $G_t= J_t\nabla_{zz}\loss\left(z_t, x\right) + \nabla_{xz}\loss\left(z_t, x\right)$.
Using $\|I - \rho \nabla_zz\loss(z, x)\|\leq \kappa$, a first crude upper bounding gives
$$
\|J_{t+1}\| \leq \kappa \|J_t\| + \alpha\enspace,
$$
where $\alpha$ is an upper-bound of $\|\nabla_{xz}\loss\|$. This shows that $\|J_t\|$ is bounded.
Next, denoting $\tilde{G}_t = J_t\nabla_{zz}\loss\left(z^*, x\right) + \nabla_{xz}\loss\left(z^*, x\right)$, we find
\begin{align*}
    \Delta^t \triangleq G_t - \tilde{G}_t =& J_t\left( \nabla_{zz}\loss\left(z_t, x\right) - \nabla_{zz}\loss\left(z^*, x\right)\right) \\&+   \nabla_{xz}\loss\left(z_t, x\right) - \nabla_{xz}\loss\left(z^*, x\right)
\end{align*}
Using the third-order differentiability of $\loss$, the rate of convergence of $z_t$, and that $J_t$ is bounded, we find that there exists $\beta> 0$ such that $\|\Delta_t\| \leq \beta \kappa_t$.
Eq.~\ref{eq:jac_rec} finally gives
\begin{align*}
    J_{t+1} - J^* =  &\left(Id - \rho \nabla_{zz}\loss(z^*, x)\right)\left(J_t - J^*\right)\\
    & - \rho \nabla_{zz}\loss(z^*, x) \Delta_t\enspace.
\end{align*}
Taking norms and using the triangular inequality, we find
$$
\|J_{t+1}-J^*\| \leq \kappa \|J_t - J^*\| + \gamma \kappa^t\enspace,
$$
where $\gamma = \rho \mu \beta$. Unrolling the recursion gives, as expected,
$$
\|J_{t}-J^*\|\leq \gamma t \kappa^{t-1}\enspace.
$$
\end{proof}

\subsection{Tightness of the bound in \autoref{prop:jac_gd}}

Importantly, the rate of $\mathcal O(t\kappa^t)$ given in \autoref{prop:jac_gd} is tight.
Indeed, it is reached with in the following example.

\begin{proposition}
\label{prop:conv_grad_autodiff}
For $x \in \bbR^n$, let $\lambda_x = \sum_{i=1}^mx_i$. Consider $\loss(z, x) = \frac12 \lambda_x\|z\|^2$. The iterates produced by gradient descent with step $\rho\leq 1 / \lambda_x$ verify $z_t = \kappa^tz_0$ with $\kappa = 1 - \rho \lambda_x$, and we have $J_t=-\rho t\kappa^{t-1} \mathbb{1}_nz_0^{\top}$.
\end{proposition}

\subsection{Proof of \autoref{prop:sgd_bound}}
\label{proof:sgd_bound}

\sgdBound*{}

\begin{proof}
Taking expectations in \autoref{prop:bound_g2} and using \autoref{eq:lip_rj} gives
$$
\bbE[|g^2- g^*|] \leq L_z \bbE[\|J_t -J^*\| |z_t-z^*|] + \frac L 2 \underbrace{\bbE[|z_t - z^*|^2]}_{d_t} \enspace.
$$
Cauchy-Schwarz on the first term gives
$$
\bbE[\|J_t -J^*\| |z_t-z^*|] \leq  \sqrt{\bbE\left[\|J_t-J^*\|^2\right]} \sqrt{d_t}
$$
and then $\|J_t-J^*\|^2\leq \|J_t-J^*\|_F^2$ gives the advertised result.
\end{proof}

\subsection{Proof of ~\autoref{prop:bounding_ineq}}
\label{proof:bounding_ineq}

\boundingIneq*{}

\begin{proof}

Let $U_t = \nabla_{zz}C(z_t, x, \xi_{t+1}) J_t + \nabla_{zx}C(z_t, x, \xi_{t+1})$. We have $\bbE_{\xi_{t+1}}[U_t] = \nabla_{zz}\loss(z^*, x) (J_t - J^*) + \Delta_t$, where $\Delta_t = \left(\nabla_{zz}\loss(z_t, x) - \nabla_{zz}\loss(z^*, x)\right)J_t + \nabla_{zx}\loss(z_t, x) - \nabla_{zx}\loss(z^*, x)$. We find
\begin{align*}
    \|J_{t+1} - J^*\|_F^2 =& \|J_t - J^*\|_F^2 \\
    &- 2\rho_t \langle J_t - J^*, U_t\rangle_F + \rho_t^2 \|U_t\|_F^2
\end{align*}
Taking expectations with respect to $\xi_{t+1}$ yields
\begin{align*}
  \bbE_{\xi_{t+1}}\big[\|J_{t+1}-&J^*\|_F^2\big] =\\& \|J_t -J^*\|_F^2 \\
  &- 2\rho_t \langle J_t - J^*, \nabla_{zz}\loss(z^*, x)(J_t - J^*)\rangle_F  \\
  &- 2 \rho_t \langle J_t - J^*, \Delta_t\rangle_F + \rho_t^2  \bbE_{\xi_{t+1}}\left[\|U_t\|_F^2\right]
\end{align*}
Using strong-convexity for the second term and Cauchy-Schwarz for the third term, we find
\begin{align*}
  \bbE_{\xi_{t+1}}\left[\|J_{t+1}-J^*\|_F^2\right] \leq& (1 - 2\rho_t \mu)\|J_t -J^*\|_F^2 \\
  &+ 2 \rho_t \|J_t - J^*\|_F\|\Delta_t\|_F \\
  &+ \rho_t^2  \bbE_{\xi_{t+1}}\left[\|U_t\|_F^2\right]
\end{align*}
Taking expectations over the whole past
\begin{align*}
  \delta_{t+1} \leq& (1 - 2\rho_t \mu)\delta_t \\
  &+ 2 \rho_t \bbE\left[\|J_t - J^*\|_F\|\Delta_t\|_F\right] + \rho_t^2  \bbE\left[\|U_t\|_F^2\right]
\end{align*}
To majorize the last term,
$$
\|U_t\|_F^2 \leq \underbrace{\|\nabla_{zz}C(z, x, \xi_{t+1})J_t\|^2_F}_{\leq\|J_t\|^2 \|\nabla_{zz}C(z, x, \xi_{t+1}\|^2_F} +\|\nabla_{xz}C(z, x, \xi_{t+1}\|^2_F
$$
Therefore
$$\bbE\left[\|U_t\|_F^2\right]\leq L_J \sigma_{zz}^2 + \sigma_{xz}^2$$

Cauchy-Schwarz on the middle term yields
\begin{align*}
\bbE\left[\|J_t - J^*\|_F\|\Delta_t\|_F\right] \leq & \sqrt{r} \sqrt{\delta_t} \sqrt{\bbE[\|\Delta_t\|^2]}\\
        \leq &  \sqrt{r}L \sqrt{\delta_t d_t}\enspace.
\end{align*}
Combining everything provides the final bound.
\end{proof}

\subsection{Proof of ~\autoref{prop:sgd_cst_step}}
\label{proof:sgd_cst_step}

\sgdCstStep*{}

\begin{proof}
We start by obtaining a simpler bound than~\ref{eq:bound_fro} by completing the squares
\begin{align*}
  (1 - 2\rho_t\mu) \delta_t + 2\sqrt{r}L \rho_t \sqrt{d_t}\sqrt{\delta_t} \leq \\ \left(\sqrt{1 - 2\rho_t\mu}\sqrt{\delta_t} + \frac{\sqrt{r}L\rho_t}{\sqrt{1 - 2\rho_t\mu}}\sqrt{d_t}\right)^2
\end{align*}

And bouding crudely $\sqrt{a+b}\leq\sqrt{a} + \sqrt{b}$, we obtain a simple recursion on $\sqrt{\delta_t}$
\begin{equation}
\label{eq:simpler_recursion}
    \sqrt{\delta_{t+1}}\leq \sqrt{1 -2\rho_t\mu}\sqrt{\delta_t} + \frac{\sqrt{r}L\rho_t}{\sqrt{1 - 2\rho_t\mu}}\sqrt{d_t} + \rho_t B\enspace.
\end{equation}
\citep{Moulines2011} give
\[
    \bbE[|z_t - z^*|^2] \leq (1 - 2 \rho \mu)^t |z_0 - z^*|^2 + \beta^2
    \enspace .
\]
Using $\sqrt{a + b} \leq \sqrt{a} + \sqrt{b}$ for $a, b\geq 0$, we get
\begin{align*}
    \sqrt{d_t} &\leq \kappa_2^t|z_0 - z^*| + \beta
\end{align*}
Eq.~\eqref{eq:simpler_recursion} then gives
$$
    \sqrt{\delta_{t+1}}\leq\kappa_2\sqrt{\delta_t} + \rho \kappa_2^t + (1 - \kappa_2)B_2
$$
Unrolling the recursion and using $\sum_{i=0}^{t} \kappa_2^i \leq \frac{1}{1-\kappa_2}$ gives the proposed bound on $\delta_t$.
\end{proof}

\subsection{Proof of~\autoref{prop:sgd_decr_step}}
\label{proof:sgd_decr_step}

\sgdDecrStep*{}

\begin{proof}

Under the assumptions, Eq.~\eqref{eq:bound_fro} becomes
$$
\delta_{t+1}\leq (1 -2 \mu C t^{-\alpha})\delta_t + 2\sqrt{r}LdCt^{-\frac32\alpha}\sqrt{\delta_t} + B^2C^2t^{-2\alpha}\enspace.
$$
We get rid of the problematic middle term using the inequality, valid for all $\chi>0$,
$$
t^{-\frac32\alpha}\sqrt{\delta_t} \leq \frac12 \left(\frac{\chi \delta_t}{t^\alpha} + \frac{1}{\chi t^{2\alpha}}\right),
$$
which gives
\begin{align*}
\delta_{t+1} \leq &\left(1 - (2\mu C - \chi \sqrt{r}LdC)t^{-\alpha}\right)\delta_t \\
&+ (B^2 C^2 + \frac{\sqrt{r}LdC}{\chi})t^{-2\alpha}\enspace.
\end{align*}
We take $\chi = \frac{\mu}{\sqrt{r}Ld}$, so that the first term becomes \mbox{$1 - \mu Ct^{-\alpha}$}.
Note that it does not give optimal rates, but makes computations much simpler.
In \citep{Moulines2011}, it is shown that for $a, b > 0$,  a recursion satisfying
$$
\delta_{t+1}\leq (1 - at^{-\alpha}) \delta_t + bt^{-2\alpha}
$$
verifies $\delta_t \leq 4 \frac{b}{a}t^{-\alpha} + o(t^{-\alpha})$. The result follows by taking $a = \mu C$ and $b=B^2 C^2 + \frac{rL^2 d^2C}{\mu^2}$.
\end{proof}

\subsection{Proof of \autoref{prop:conv_pth}}
\label{proof:conv_pth}

\convPth*{}

\begin{proof}
Gradient descent iterates
$$z_{t+1}=z_t - \rho D^{\top} (Dz_t - x)^{p-1}\enspace.$$
The residuals $r_t = x - Dz_t$ therefore verify the recursion
$$
r_{t+1} = r_t - \rho DD^{\top}r_t^{p-1}\enspace.
$$
Since we assume $DD^{\top} = I_n$, they verify $r_{t+1} = r_t - \rho r_t^{p-1}$. Each entry of $r_t$ therefore evolves independently, following the $1$-d recursive equation
\begin{equation}
    \label{eq:1d_rec}
    u_{t+1} = u_t - \rho u_t^{p-1}
\end{equation}
Standard analysis techniques show that this gives $u_t = \left(\frac{1}{\rho (p-2)t}\right)^{\frac{1}{p-2}}(1 + O(\frac{\log(t)}{t}))$, and therefore each coefficient of $r_t$ satisfies the same asymptotic development.
The Jacobian verifies
$$
J_{t+1} = J_t - (p-1)\rho (J_t D^{\top} - I_m)\diag( r_t^{p-2})D,
$$
and denoting $M_t =  I_n-J_tD^{\top}$, we find
$$
M_{t+1} = M_t(I_n - (p-1)\rho \diag(r_t^{p-2}))\enspace.
$$
Since the rightmost term is diagonal, we can rewrite this recursion as:

$$
M_{t+1} = M_t\diag(\mathbb{1}_n - (p-1)\rho r_t^{p-2})
$$
Unrolling this recursion gives:
$$
M_t = M_0 \diag(\prod_{j \leq t-1} \mathbb{1} - (p-1)\rho r_t^{p-2})
$$
We can then majorize each coefficient in the $\diag$ by:
$$
\prod_{j \leq t-1}(1 - (p-1)\rho u_j^{p-2}) \leq \exp(\sum_{j \leq t-1} - (p-1) \rho u_j^{p-2})\enspace,
$$
where $u_t$ follows the recursion~\eqref{eq:1d_rec}.
The asymptotic development of $u_t$ gives:
$$
u_t^{p-2} = \frac{1}{\rho(p-2)t} + O(\frac{\log(t)}{t^2})
$$
and as a consequence, denoting $\alpha = \frac{p-1}{p-2}$:
$$
\exp(\sum_{j \leq t-1} - (p-1) \rho u_j^{p-2}) = O( t^{-\alpha})
$$
Overall, we have $M_t = O(t^{-\alpha})$ and $r_{t}^{p-1} = O(t^{-\alpha})$, so $g_2 = M_tr_t^{p-1} = O(t^{-2\alpha}$.
\end{proof}

\section{Proof of \autoref{prop:inexact_sgd}}
\label{proof:inexact_sgd}

We start by giving the convergence rate of the SGD with $(\delta, L, \mu)$-inexact oracle for a function $f$ defined on $x\in\bbR^n$ as
\[
    f(x) = \bbE_\upsilon[F(x, \upsilon)]
\]
for $\upsilon$ a random variable distributed with probability $d_\upsilon$.
For $\upsilon_0 \sim d_\upsilon$, we denote $(F_\delta(\cdot, \upsilon_0), G_\delta(\cdot, \upsilon_0)$ a $(\delta, L, \mu)$-inexact oracle  of $F(\cdot, \upsilon_0)$, uniform in $\upsilon_0$.

\begin{restatable}{lemma}{inexactSgd}
    \label{lem:inexact_sgd}
    For a $\mu$-strongly convex $L$-smooth function $f$ and a $(\delta, \mu, L)$-inexact oracle $(F_\delta, G_\delta)$ of $F$ such that
    \begin{align*}
        \mathbb E_\upsilon[F_{\delta}(x, &\upsilon)] = f_\delta(x)
       \:\text{ , }\:
        \mathbb E_\upsilon[G_{\delta}(x, \upsilon)] = g_\delta(x) \enspace ,\\
       \:\text{ and }~~~~&
        \mathbb E_\upsilon[|G_{\delta}(x, \upsilon) - g_\delta(x)|^2] \le \sigma^2
        \enspace .
    \end{align*}
    Then, the iterates of the stochastic gradient descent with constant step-size $\eta < \frac{1}{L}$ verify
    \[
        \mathbb E |x_q - x^*|^2 \le (1 - \eta\mu)^q |x_0 - x^*|^2 + \frac{\eta}{\mu} \sigma^2 + \frac{2}{\mu} \delta.
    \]
\end{restatable}

\begin{proof}

    Consider the solution estimate $x_{q+1}$ at iteration $q$, obtained through stochastic gradient descent \emph{i.e.}
    $x_{q+1} = x_q - \eta\nabla_x G_{\delta}(x_q, \upsilon_{q+1})$. We denote $r_{q+1} = \mathbb E [|x_{q+1} - x^* |^2]$ and $\widehat r_{q+1} = \mathbb E [|x_{q+1} - x^* |^2 | \upsilon_{q+1}]$.
    Then
    \begin{align*}
        |x_{q+1} - x^*|^2 = & |x_{q} - x^*|^2 - 2\eta \langle G_\delta(x_q, \upsilon_{q+1}), x_q - x^*\rangle\\
            & + \eta^2|G_\delta(x_q, \upsilon_{q+1})|^2\\
    \end{align*}
    We take the expectation relatively to $\upsilon_{q+1}$
    \begin{align*}
        \widehat r_{q+1} \le & |x_{q} - x^*|^2 - 2\eta \langle g_\delta(x_q), x_q - x^*\rangle \\
        & + \eta^2|g_\delta(x_q)|^2 + \eta^2\mathbb E_{\upsilon_q}[|G_{\delta}(x_q, \upsilon_q) - g_\delta(x_q)|^2] \\
        \le & |x_{q} - x^*|^2 + 2\eta \langle g_\delta(x_q), x^* - x_q\rangle \\
        & + \eta^2|g_\delta(x_q)|^2 + \eta^2\sigma^2
    \end{align*}
    Using $\langle g_\delta | x^* - x_q\rangle \le f(x^*) - f_{\delta}(x_q) - \frac{\mu}{2}|x_q - x^*|^2$, we get

    \begin{align*}
        \widehat r_{q+1}^2 \le & |x_{q} - x^*|^2 + 2\eta (f(x^*) - f_{\delta}(x_q) - \frac{\mu}{2}|x_q - x^*|^2) \\
        & + \eta^2|g_\delta(x_q)|^2 + \eta^2\sigma^2\\
        \le & (1 - \eta\mu)|x_{q} - x^*|^2 + \eta^2\sigma^2\\
        & + 2\eta (f(x^*) - f_{\delta}(x_q) + \frac{\eta}{2}|g_\delta(x_q)|^2)
    \end{align*}
    We introduce $\widetilde x = x_q - \eta g_{\delta}(x_q)$. Using the rhs of the $\delta$-inexact oracle, we have
    \begin{align*}
        &f(\widetilde x) - f_\delta(x_q) - \langle g_\delta(x_q) | \widetilde x - x_q\rangle \le \frac{L\eta^2}{2}|g_\delta(x_q)|^2 + \delta\\
        &\emph{i.e}\\
        &f(\widetilde x) - f_\delta(x_q) + \frac{\eta}{2}|g_\delta(x_q)|^2
        \le \delta - (1 - \eta L) \frac{\eta}{2} |g_\delta(x_q)|^2
        \enspace .
    \end{align*}
    Using this in the previous equation, we obtain

    \begin{align*}
        r_{q+1}^2 \le & (1 - \eta\mu)|x_{q} - x^*|^2 + \eta^2\sigma^2 + 2\eta\delta\\
        &  + \underbrace{2\eta (f(x^*) - f(\widetilde x))
          - (1 - \eta L) \eta^2 |g_\delta(x_q)|^2}_{\le 0}
          \enspace ,
    \end{align*}
    where the two terms on the second line are non-positive. Indeed, $\eta \le \frac{1}{L}$ and $f(x^*) \le f(\widetilde x)$.
    Taking the expectation relatively to $\upsilon_{0}, \dots, \upsilon_{q-1}$ gives the following recursion relationship
    \[
        r_{q+1}^2 \le (1 - \eta\mu)r_{q}^2
            + \eta^2\sigma^2 + 2\eta\delta
            \enspace .
    \]
    applying this recursion $q$ times yields
    \[
        r_{q+1}^2 \le (1 - \eta\mu)^q r_{0}^2
            + (\eta^2\sigma^2 + 2\eta\delta)\sum_{k=0}^q (1 - \eta\mu)^k
            \enspace ,
    \]
    and we obtain the desired results as $\sum_{k=0}^q (1 - \eta \mu) < \frac{1}{\eta\mu}$
\end{proof}

\inexactSgdAutodiff*{}

\begin{proof}
As shown in \autoref{sub:inexact_oracle}, $(\ell, g^i)$ is a $(\widetilde\delta_i, \frac{\mu_x}{2}, 2L_x)$-inexact oracle with $\widetilde\delta_i = \Delta_i^2(\frac{1}{\mu_x} + \frac{1}{2L_x})$.

The variance of the stochastic inexact oracles can be bounded as follow
\begin{align}
    \nonumber
    \mathbb E [|g^i(x, \upsilon) - g^i(x)|^2]
        \le & 2(E [
        |g^i(x, \upsilon)  - \nabla_x h(x, \upsilon)|^2
        \\
        & + |\nabla_x h(x, \upsilon) - \nabla_x \ell(x)|^2\\\nonumber
        & + |\nabla_x\ell(x)- g^i(x)|^2])
        \enspace ,\\
        \le & \underbrace{2(\sigma^2 + 2\Delta_i^2)}_{\widetilde\sigma^2}
        \enspace .
\end{align}
Using these two bounds with the previous result yield
\begin{align}
    \mathbb E |x_q - x^*|^2
        \le &
        (1 - \frac{\eta\mu}{4})^q |x_0 - x^*|
        + \frac{4\eta}{\mu} \widetilde\sigma^2
        + \frac{2}{\mu} \widetilde\delta\\
        \le &
        (1 - \frac{\eta\mu}{4})^q |x_0 - x^*|
        + \frac{4\eta}{\mu} \sigma^2
        + \frac{2}{\mu} \delta
\end{align}
    with $\delta = \Delta_i^2 (\frac{1}{L} + \frac{1}{\mu} + 2\eta)$

\end{proof}

\end{document}